\documentclass[twoside]{article}

\usepackage[accepted]{arxiv_dc}
\usepackage[round]{natbib}

\usepackage{amsmath}

\usepackage{amssymb}
\usepackage{amsthm}

\usepackage{microtype}
\usepackage{graphicx}
\usepackage{booktabs} % for professional tables
\usepackage{natbib}
\usepackage{xcolor}
\usepackage{natbib}
\usepackage{hyperref}
\usepackage{url,bm}

\usepackage{amsthm,amsfonts,amsmath,amssymb,epsfig,color,float,graphicx,verbatim}
\usepackage{algorithm,algorithmic}
\usepackage{bbm}
\usepackage{enumitem}
\usepackage{nicefrac}
\usepackage{xfrac}
\usepackage{graphicx}
\usepackage{subcaption}

\usepackage{array}

\newtheorem{theorem}{Theorem}[section]

\newtheorem{lemma}{Lemma}[subsection]

\newtheorem{remark}{Remark}[section]
\newtheorem{example}{Example}

\newcommand{\subsecref}[1]{Section~\ref{#1}}

\renewcommand{\eqref}[1]{Eq.~(\ref{#1})}

\usepackage{amssymb}

\usepackage{bbm}
\newcommand{\onefunc}{\mathbbm{1}}

\newcommand{\stam}[1]{}
\newcommand{\ignore}[1]{}
\usepackage{comment}

\usepackage{mathtools}

\DeclarePairedDelimiter\floor{\lfloor}{\rfloor}

\newcommand{\on}[1]{#1_{\shortparallel}}
\newcommand{\off}[1]{#1_{\perp}}
\newcommand{\flt}[1]{#1_{\vdash}}
\newcommand{\bld}[1]{\mathbf{#1}}

\newcommand{\wt}{\Tilde{w}}
\newcommand{\At}{\Tilde{\bld{A}}}
\newcommand{\at}{\Tilde{a}}

\newcommand{\cL}{\mathcal{L}}

\newcommand{\bE}{\mathbbm{E}}
\newcommand{\inner}[2]{ #1\odot #2}
\newcommand{\Inner}[2]{ \langle #1 \; , \; #2 \rangle}
\newcommand{\iter}[2]{ #1^{(#2)}}
\DeclareMathOperator{\vect}{vec}

\newcommand{\bU}{\mathbf{U}}

\newcommand{\bI}{\mathbf{I}}

\newcommand{\ca}{{\cal A}}

\newcommand{\cd}{{\cal D}}

\newcommand{\cw}{{\cal W}}
\newcommand{\cl}{{\cal L}}

\newcommand{\cu}{{\cal U}}

\DeclareMathOperator*{\E}{\mathbb{E}}

\newcommand{\norm}[1]{\left\|#1\right\|}
\newcommand{\snorm}[1]{\|#1\|} %small norm

\begin{document}

% If your paper is accepted and the title of your paper is very long,
% the style will print as headings an error message. Use the following
% command to supply a shorter title of your paper so that it can be
% used as headings.
%
%\runningtitle{I use this title instead because the last one was very long}

% If your paper is accepted and the number of authors is large, the
% style will print as headings an error message. Use the following
% command to supply a shorter version of the authors names so that
% they can be used as headings (for example, use only the surnames)
%
%\runningauthor{Surname 1, Surname 2, Surname 3, ...., Surname n}

\twocolumn[

\aistatstitle{Adversarial Vulnerability due to On-Manifold Inseparability}

\aistatsauthor{ Rajdeep Haldar \And Yue Xing \And  Qifan Song \And Guang Lin}

\aistatsaddress{ Purdue University \And  Michigan State University \And Purdue University \And Purdue University} ]

\begin{abstract}
  Recent works have shown theoretically and empirically that redundant data dimensions are a source of adversarial vulnerability. 
  However, the inverse doesn't seem to hold in practice; employing dimension-reduction techniques doesn't exhibit robustness as expected. In this work, we consider classification tasks and
  characterize the data distribution as a low-dimensional manifold, with high/low variance features defining the on/off manifold direction.
  We argue that clean training experiences poor convergence in the off-manifold direction caused by the ill-conditioning in widely used first-order optimizers like gradient descent. The poor convergence then acts as a source of adversarial vulnerability when the dataset is inseparable in the on-manifold direction. We provide theoretical results for logistic regression and a 2-layer linear network on the considered data distribution. Furthermore, we advocate using second-order methods that are immune to ill-conditioning and lead to better robustness. We perform experiments and
  exhibit tremendous robustness improvements in clean training through long training and the employment of second-order methods, corroborating our framework. Additionally, we find the inclusion of batch-norm layers hinders such robustness gains. We attribute this to differing implicit biases between traditional and batch-normalized neural networks.
\end{abstract}

\section{Introduction}
Neural networks exhibit high classification performance and generalize well for test examples drawn from the data distribution exposed during training\citep{liu2020understanding}. However, they showcase surprising vulnerability to imperceptible perturbations to the input data, hampering the classification performance severely. The perturbed inputs exploiting the vulnerability or attacking the originally trained model are known as \emph{adversarial examples} in the literature \citep{szegedy2013intriguing,goodfellow2014explaining,madry2017towards,carlini2017towards}.

In recent years, there has been tremendous progress on the empirical side in developing state-of-the-art attacks and algorithms to defend against adversarial examples; unfortunately, the theoretical foundations explaining the existence of adversarial attacks haven't enjoyed the same pace. Understanding the founding mechanism is critical to developing inherently robust models that naturally align with the human decision-making process and can highlight the fundamental flaws in modern-day model training.

%There is a conflicting understanding in the literature regarding the adversarial vulnerability of neural networks. On one hand, based on the manifold literature and implicit bias literature, off-manifold dimension is the source of the vulnerability of the model, and the maximum margin classifier based on the on-manifold dimension is robust against adversarial attacks. 
The manifold hypothesis, suggesting that the support of underlying data is a low-dimensional manifold lying in a higher-dimension ambient space, is the most promising theory pertaining to the existence of adversarial examples. Works like \cite{melamed2024adversarial,pmlr-v238-haldar24a} show empirically and theoretically that with a large number of redundant dimensions (e.g., useless background pixels for image tasks), one can generate adversarial examples of arbitrarily small magnitude. In addition, the implicit bias literature \citep{wei2019regularization,lyu2019gradient,lyu2021gradient,nacson2022stochastic} suggests clean training leads to a model that is a KKT solution of the \emph{maximum margin problem}. More generally, \cite{rosset2003margin} shows that the widely used logistic loss is a \emph{maximizing margin loss}, which extends to deep neural networks. 

The notion of a geometrically robust maximum margin classifier is inherited from the theory of Support Vector Machines (SVM). For linear models, the geometric notion of the maximum margin classifier is equivalent to solving the max-margin problem in implicit bias literature. Although the two definitions are not directly comparable for neural networks (non-linear networks), there is an equivalence of sufficiently wide neural networks and SVMs or Kernel Machines (KM) in general, accompanied by robustness certificate results \citep{chen2021equivalence,domingos2020every}. 
Based on the equivalence results and visual boundary of neural network classifiers in our experiments (Fig: \ref{fig:slow_convergence}), we expect the max-margin implicit bias of neural networks to enforce robustness and be immune to small magnitude imperceptible attacks. More discussion can be found in the third paragraph of Section \ref{sec:related}.

The prior results suggest that a cleanly trained model should exhibit robustness in the absence of redundant dimensions due to its natural inclination toward learning a max-margin classifier.
However, in practice, implementing dimension reduction techniques like PCA does not improve model robustness \citep{alemany2020dilemma,aparne2022pca}. Hence, there is a gap between what we expect and observe.

This work bridges the gap between the theory suggesting robust models in the absence of redundant dimensions and the persistent vulnerability observed in practice. The implicit bias and large margin properties of logistic loss hold at \emph{convergence} to the optimal classifier. We claim that, in practice, convergence isn't attained due to \emph{ill-conditioning} \citep{Goodfellow-et-al-2016,boyd2004convex,nesterov2013introductory} when dealing with first-order optimization methods like \emph{gradient-descent}, which adds to the vulnerability of our trained model. Throughout this paper, the default underlying loss for our arguments is the logistic loss or its multivariate counterpart, the cross-entropy loss, as it is the most natural and popular choice of loss for classification problems.

We generalize the idea of off/on-manifold dimensions beyond redundant/useful dimensions as suggested in \cite{pmlr-v238-haldar24a,melamed2024adversarial} to low/high variance features.
Even if the data is separable, if it is non-separable in the on-manifold dimensions, the optimal classifier is dependent on the off-manifold dimensions. However, convergence is significantly slower in the off-manifold direction (ill-conditioning), resulting in a suboptimal solution that is not large/maximum margin, leading to adversarial vulnerability.

Figure \ref{fig: real-life example} motivates our framework with an example. Consider a classification problem between birds and insects.  Most birds and insects can be distinguished based on wing length (high variance/on-manifold); birds, in general, have much larger wing lengths than insects with small or no wings. However, some insects might have wings larger than the smallest birds, like hummingbirds; hence, we can use the presence of a beak as an additional discrete feature (i.e., low-variance/off-manifold direction) to distinguish them further. Note that the data distribution of birds/insects is separable using both beak presence and wing length, but it isn't separable solely based on wing length. Figure \ref{fig: theoretical model} shows that even though the optimal classifier (attained at convergence) is robust, the estimated decision boundary learned due to lack of convergence in the off-manifold direction is not robust. This is caused by the ill-conditioned nature of the training, which is related to the extent of low-dimensionality, which is characterized by the ratio of on/off-manifold variances.

     \begin{figure}[h]
         \centering
         \includegraphics[width=0.45\textwidth]{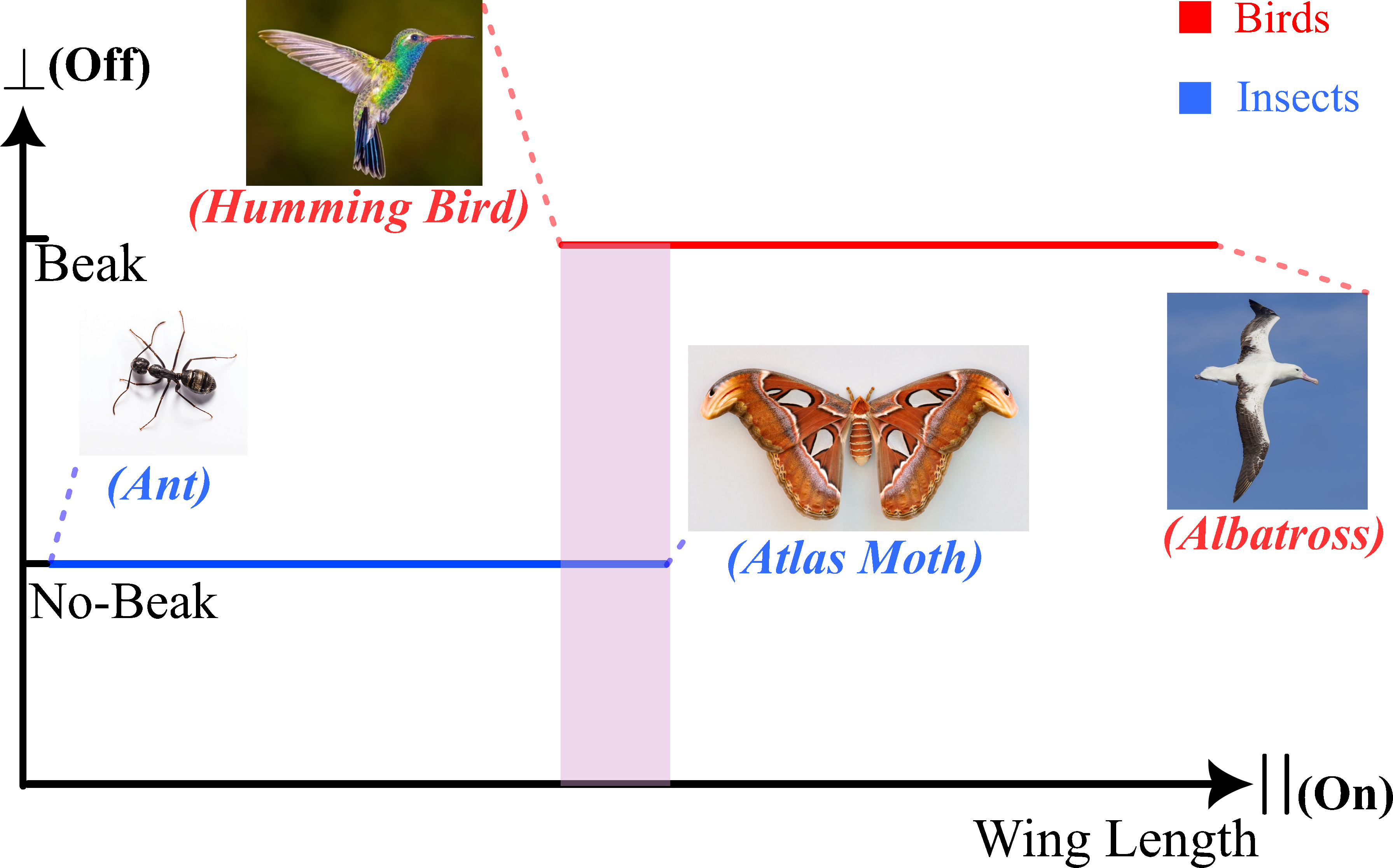}
         \caption{Binary classification between birds and insects. The purple region represents overlap in the on-manifold feature, where only the off-manifold feature can distinguish between the two classes.}
         \label{fig: real-life example}
     \end{figure}
     \begin{figure}[h]
         \centering
         \includegraphics[width=0.35\textwidth]{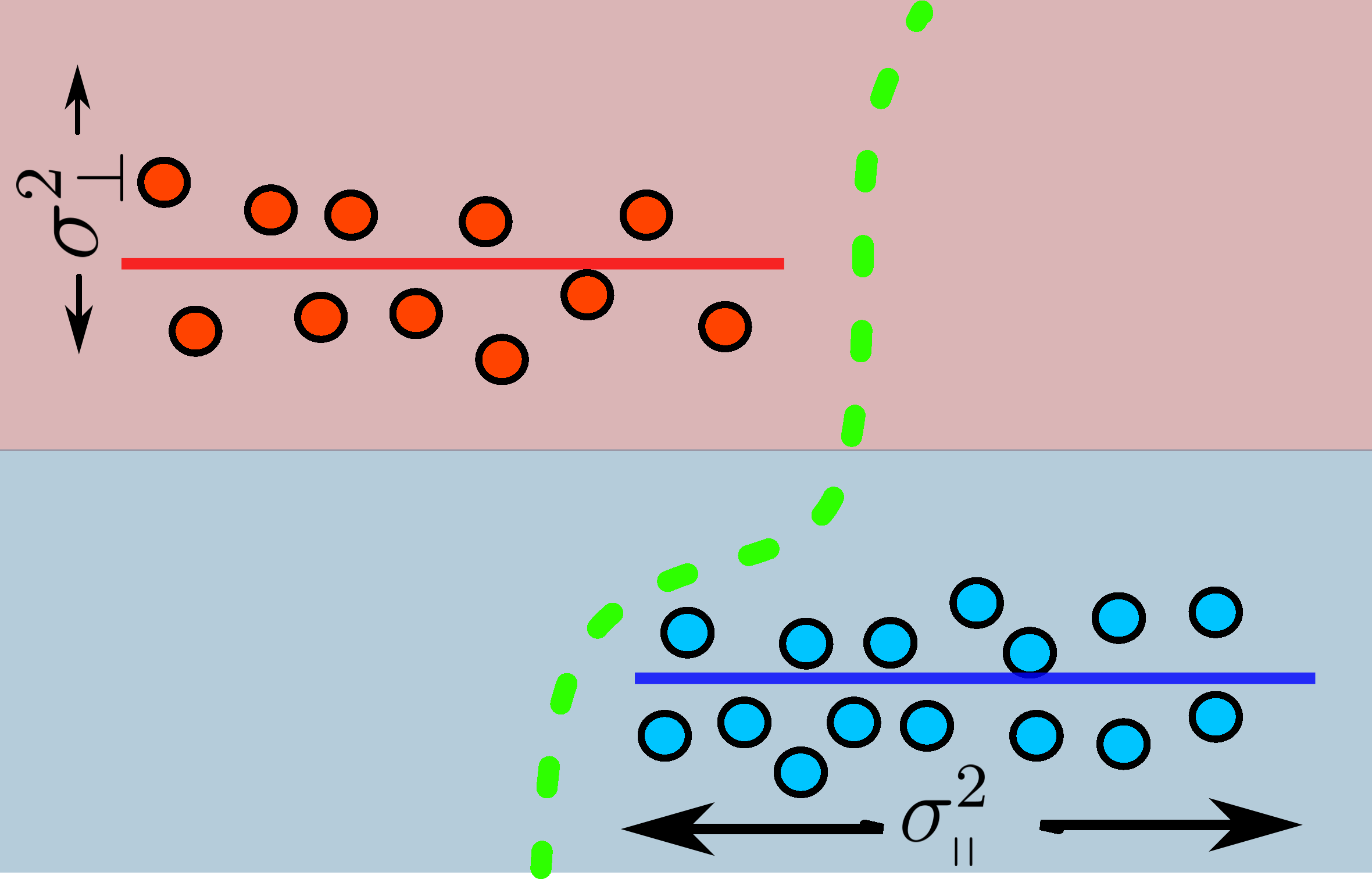}
        \caption{Optimal classifier robustly separates the ambient space into red and blue regions. The estimated decision boundary (green) is suboptimal and vulnerable even though it separates the data manifold accurately. %The low-dimensional nature of the data manifold is characterized by $\nicefrac{\on{\sigma}}{\off{\sigma}}\to 0$. %\textcolor{red}{maybe we can incorporate notations $k$,$l$ etc in the figure to better illustrate Section 3.1}
        }
        \label{fig: theoretical model}
     \end{figure}

\begin{figure}[ht]
    \centering
    \includegraphics[width=0.5\textwidth]{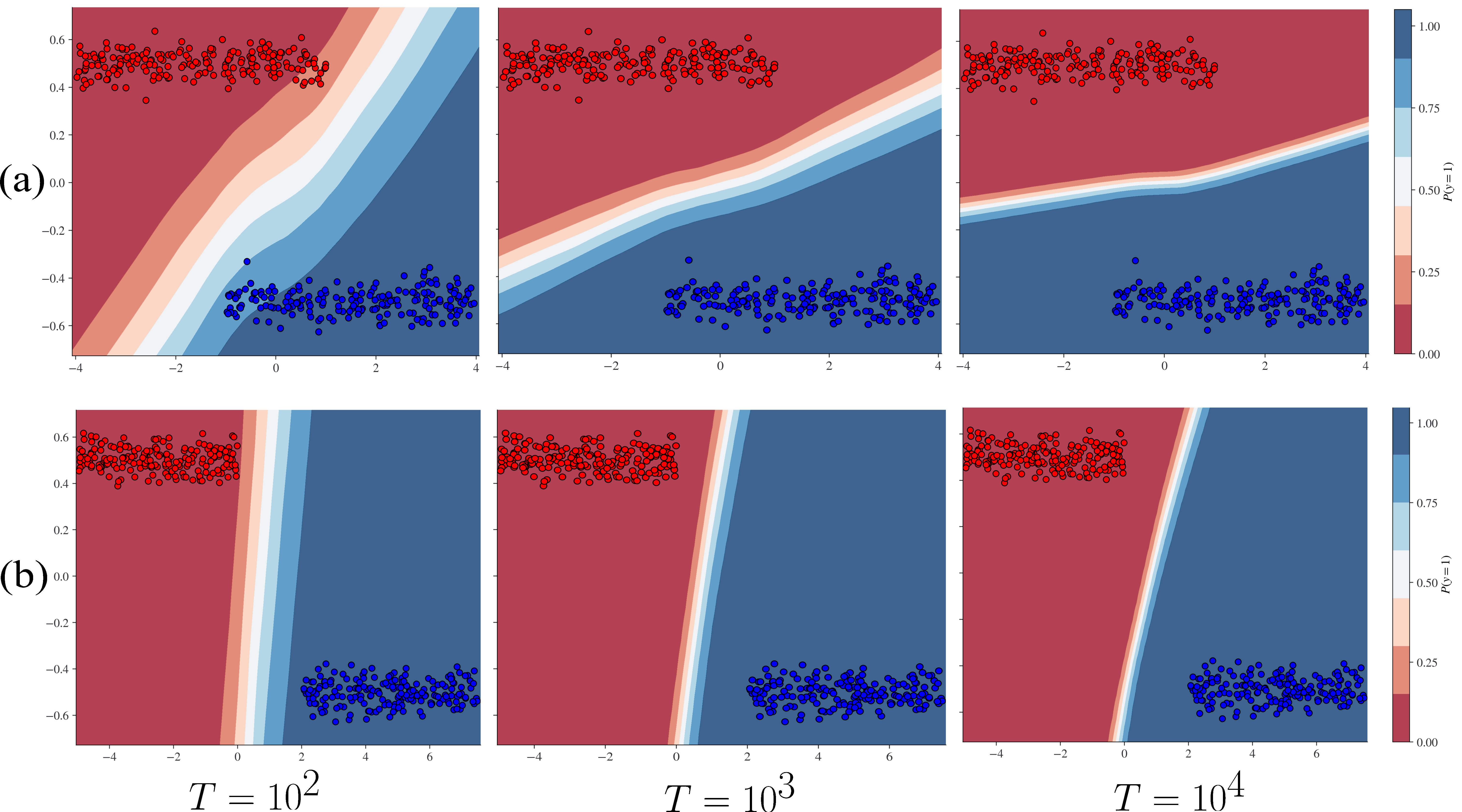}
    \caption{Neural Network estimated decision boundary for $T$ training epochs. Training and testing accuracy is 100\% for all $T\geq 10^2$. Data distribution is on-manifold (a)  inseparable (b) separable.} %: The shape, in turn, robustness of the decision boundary still improves over time due to slow convergence in the off-manifold direction. (b) On-manifold separable: The shape of optimal decision boundary is solely dependent on on-manifold direction and is stable due to fast convergence.}
    \label{fig:slow_convergence}
\end{figure}

To further illustrate the crucial concepts of on/off manifold and low-dimensionality, we give two mathematical examples below
\begin{example}
  $x_1\sim N(0,\on\sigma^2), x_2\sim N(0,\off\sigma^2)$ be distribution of two features, as $\on\sigma\approx \off\sigma$ 
 the data resembles a 2-D disk. However, as $\sfrac{\off\sigma}{\on\sigma}\to 0$, the distribution converges to a 1-D line. The one-dimensional direction is dictated solely by the high-variance feature $x_1$ or the on-manifold feature.
\end{example}
\begin{example}
    Consider a 3-D sphere described in polar coordinates ($R,\phi,\psi$). As $\sigma_R<<\sigma_\phi\simeq\sigma_\psi$ the data distribution resembles a 2-D shell, with on-manifold direction dictated by the two angles $\psi,\phi$
 and off-manifold being $R$
 direction. Similarly, if both $\sigma_R,\sigma_\phi<< \sigma_\psi$
 , then the sphere becomes a 1-D circle, with on-manifold direction dictated only by $\psi$. The ratio of the high-variance/low variance  determines the extent of dimensionality.
\end{example}

We summarize our key contributions as the following informal theorems:
\begin{theorem}[Informal version of Theorem \ref{thm: parameter convergence}]
    Convergence to the optimal parameter is faster and independent of dimensionality in the on-manifold direction compared to the off-manifold direction. Furthermore, as dimensionality reduces, the convergence rate for the off-manifold direction worsens.
\end{theorem}
\begin{theorem}[Informal version of Theorem \ref{thm: loss convergence}] There is an error threshold, determined by the separability in the on-manifold direction, which can be quickly reached, but reducing error below that threshold is slower and depends on data dimensionality. 
\end{theorem}
When data is inseparable in on-manifold direction, Figure \ref{fig:slow_convergence}(a) illustrates the improvement in shape and robustness of the decision boundary over time due to slow convergence in the off-manifold direction. Figure \ref{fig:slow_convergence}(b) shows when data is on-manifold separable, the shape of the optimal decision boundary solely depends on the on-manifold direction and is stable due to fast convergence. %This aligns with our parameter/loss convergence discussions in section \ref{subsec: theorems}.
%visually illustrates the slow convergence behavior. %behavior mentioned in the above theorems. 
%Even if the model attains 100\% accuracy and almost zero loss at $T=100$, the decision boundary slowly improves in the off-manifold direction with extensive training.
 
 It is known that second-order optimization methods like Newton's method are immune to ill-conditioning. Therefore,  we promote using second-order methods (Section \ref{sec: second order}) to guarantee fast convergence and, in turn, enjoy robustness gains in cleanly trained models. We conduct experiments to support our claim (Section \ref{sec: experiments}).
\section{Related Works}\label{sec:related}
\paragraph{Adversarial Vulnerability Explained via Manifold Hypothesis:}\cite{stutz2019disentangling} introduced the distinction between \emph{on/off-manifold adversarial examples} and their role in robustness and generalization. \cite{shamir2021dimpled} showed empirically that clean-trained classifiers align with a low dimensional \emph{dimple-manifold} around the data, making them vulnerable to imperceptible perturbations. \cite{zhang2022manifold} decomposed the adversarial risk into geometric components motivated by manifold structure.  \cite{pmlr-v238-haldar24a,melamed2024adversarial} theoretically establish a relationship between the \emph{dimension gap} (induced by low-dimension data manifold) and adversarial vulnerability.
\paragraph{Attack/Defense systems based on manifold hypothesis:}\cite{xiao2022understanding} uses a generative model to learn manifold and create on-manifold attacks. \cite{jha2018detecting,lindqvist2018autogan,lin2020dual} have tried to learn the underlying manifold to detect and defend adversarial examples outside the data manifold.\\\\ 
To the best of our knowledge, our work is the first: Extending the notion of on/off manifold dimensions in terms of feature variances rather than utility; Drawing connections between on-manifold separability, ill-conditioning, and adversarial vulnerability. %Exhibiting robustness attainability in clean training.
\paragraph{Implicit Bias and Robustness} 
\cite{frei2024double} show that max-margin implicit bias of neural networks aren't the most robust model. However, the order of attack strength in this work is comparable to the signal strength, and reduces to natural/on-manifold attack scenario mentioned in \cite{pmlr-v238-haldar24a}. These attacks are perceptible by humans and relatively large magnitude. We still expect a natural robustness to imperceptible attacks or small magnitude attacks in absence of redundant dimension, based on prior discussion. \cite{min2024can} shows that under poly-RelU activation implicit bias can be robust to larger attacks.
\section{Problem Setup}
\label{sec: Problem setup}
This section describes the technical setting and assumptions used for our main results in section \ref{subsec: theorems}.
%{\color{purple} Please write one or two sentences to explain the general purpose of the whole section.}
\subsection{Notation}
 Throughout this paper, we will use the subscripts $\shortparallel, \perp$ to denote mathematical objects corresponding to on-manifold and off-manifold, respectively. The notation $\onefunc_d$ denotes the concatenated vector of ones of length $d$. For any two vectors $u,v$ such that $u_i\leq v_i $ $\forall i$, we can define a hypercube $[u,v]$ such that if $z\in [u,v]$ then for each $i$, $z_i\in [u_i,v_i]$. Denote $\bI_d$ as the identity matrix of $d$ dimensions.  Notation $\Omega(\cdot)$ is the usual asymptotic lower bound notation. For $a,b \in \mathbbm{R}$, $a\preceq b \iff a\leq b\cdot c$ for some $c>0$. Note that the notation $\preceq$ also works for negative sequences $a$ and $b$ (i.e., $a\preceq b$ implies $|b|\preceq |a|$ when $a$ and $b$ are negative). Denote $\odot$ as the Hadamard product.
 
\subsection{Assumption on Data Distribution}
\label{subsec: data dist}
We denote the underlying signal of the data with $x\in \mathbbm{R}^D$ and its corresponding label as $y\in \{-1,+1\}$. We study a binary classification problem with the data pair $(x,y)$.

Our signal can be decomposed as $x=(\on{x},\off{x})$, where $\on{x}\in \mathbbm{R}^d$ and $\off{x}\in \mathbbm{R}^g$ are the on/off manifold components respectively. Corresponding to each class, define $d$-dimensional hypercubes of side $l$ as $\mathcal{\on{I}}^{(-1)}=[-(l-k)\cdot \onefunc_d,k\cdot \onefunc_d]$ and $\mathcal{\on{I}}^{(+1)}=[-k\cdot \onefunc_d,(l-k)\cdot \onefunc_d]$. The parameter $k$ controls the overlap between the two hyper-cubes in the on-manifold direction. We assume that conditioned on the label the on-manifold signal is uniformly drawn from such hyper-cubes $\on{x}|y\sim \cu(\mathcal{\on{I}}^{(y)})$ with means $\on{\mu}^{(y)}=\frac{y\left(l-2k\right)}{2}\cdot\onefunc_d$ and covariance $\on{\sigma}^2\cdot\bI_d=\frac{l^2}{12}\cdot\bI_d$ respectively. Similarly, define non-overlapping symmetric hyper-cubes $\mathcal{\off{I}}^{(y)}=[\off{\mu}^{(y)}-\sqrt{3}\off{\sigma}\cdot\onefunc_g,\off{\mu}^{(y)}+\sqrt{3}\off{\sigma}\cdot \onefunc_g]$ such that $\mathcal{\off{I}}^{(-1)}\cap\mathcal{\off{I}}^{(+1)}=\emptyset$, and $\off{x}|y\sim \cu(\mathcal{\off{I}}^{(y)})$ with means $\off{\mu}^{(y)}$ and covariance $\off{\sigma^2}\cdot \bI_g$. The class probabilities themselves are binomial with probability $\pi$ i.e. $P(y=1)=\pi$ and $P(y=-1)=1-\pi$. %Also, the variance ratio $\frac{\off{\sigma}}{\on{\sigma}}\to 0$. 
Also, $\snorm{\off{\mu}^{(y)}}<\infty$ and $\nicefrac{\off{\sigma}}{\on{\sigma}}<1$.

% \subsection{Data Interpretation}
To explain the above assumption, the data is essentially uniformly distributed over two high-dimensional rectangles corresponding to each class. Due to the non-overlapping off-manifold distributions, the two rectangles are linearly separable in the ambient $D$-dimensional space. However, there is an overlap in the on-manifold distribution controlled by $k$. The $x$ and $\off{x}$ distributions are linearly separable, but $\on{x}$ is not for $k\neq0$. The ratio of the variances $\nicefrac{\off{\sigma}^2}{\on{\sigma}^2}<1$ characterizes the low-dimensional manifold structure. %The variance controls the thickness of the intervals in each direction; as the thickness in the off-manifold direction gets smaller, the underlying distribution becomes more like a $d$-dimensional object suspended in a larger $D$ dimensional space. 
Our data model is a mathematical representation of Figures \ref{fig: theoretical model}, \ref{fig:slow_convergence}.% Finally, as the manifold length is solely determined by $\on{\sigma}$, we assume that the means of the class distributions are comparable to the manifold length.\\

\paragraph{Overlapping coefficient}
\label{sec:ovl}
In order to formulate the extent of non-separability in the on-manifold distribution we borrow the concept of \emph{overlapping coefficient} (OVL) from traditional statistics. For any two probability densities $f_1(x),f_2(x)$ the overlapping coefficient is defined as $\nu=\int_{\cd} \min \left(f_1(x),f_2(x)\right)\, dx$ where $\cd$ is the support. Note that $\nu\in [0,1]$, and essentially represents the probability of $x$ being drawn from the minority distribution. In the context of a naive Bayes classifier, $\nu$ represents the probability of misclassification or area of conflict. We can quantify non-separability for the on-manifold distribution by computing $\nu=\int \min \left(\cu(\mathcal{\on{I}}^{(-1)}),\cu(\mathcal{\on{I}}^{(+1})\right)\, dx=\left(\nicefrac{k}{l}\right)^d$. As $k$ increases, geometrically, our on-manifold hypercubes overlap more, which is consistent with $\nu$. Later, we will see how the non-separability of the on-manifold component affects our classifier learning. 

% \begin{remark}
%    The notions of separability using OVL and low-dimensionality ($\nicefrac{\off{\sigma}^2}{\on{\sigma}^2}$) characterized by variance ratio are the keys to justify all our results. In this paper, we work with a simple uniform distribution for mathematical convenience, however, one can easily extend these ideas to any distribution. For instance, one could get identical results for a choice of Gaussian densities characterized by some overlap in the on-manifold distribution. The OVL in that case turns out to be the Cohan's D or Mahalanobi's distance. 
% \end{remark}
\subsection{Models}
\label{subsec: model setup}
For our binary classification problem, we work with the logistic loss $\ell(z)=\ln(1+e^{-z})$. We will denote $f(x,\gamma)$ with real output as the score predicted by our classifier for a particular data point $x$ and parameter vector $\gamma$. The expected loss of the classifier over the data distribution is $\cl(\gamma)=\E\limits_{\sim x,y}\ell(y\cdot f(x,\gamma))$. The clean training is the optimization problem $\min\limits_{\gamma\in\Gamma} \cl(\gamma)$ where $\Gamma$ is a compact parameter space. Let $\gamma^*=\arg\min\limits_{\gamma\in\Gamma} \cl(\gamma)$, then working with compact space ensures $\snorm{\gamma^*}$ is finite and makes analysis tractable.
This work considers the linear model and the two-layer linear network.
\subsubsection{Logistic Regression}
\label{subsec: logistic setup}
In the linear setup, $\gamma=\theta$ where $\theta\in \mathbbm{R}^D$ is the coefficient vector of logistic regression, $f(x,\theta)=\theta^Tx$. We implement a first-order gradient descent optimization scheme with step size $\alpha$ to obtain the minimizer in this scenario. The $t^{th}$ iteration step is as follows:
\begin{align}
    \theta^{(t+1)}=\theta^{(t)}-\alpha\nabla_{\theta}\cl(\theta^{(t)})
    \label{eqn: GD logistic}
\end{align}
\paragraph{Parameter components} Corresponding to $\on{x},\off{x}$ we have $\on{\theta},\off{\theta}$ such that $\theta^Tx=\on{\theta}^T\on{x}+\off{\theta}^T\off{x}$. Naturally, the notion of on/off manifold components translates to $\theta=(\on{\theta},\off{\theta})$.  
\subsubsection{Two Layer linear network}
\label{subsec: 2layer setup}
For the linear network, we over-parametrize the logistic regression case with $\theta=\bld{A}^Tw$, where $\bld{A}\in \mathbbm{R}^{m\times D}$ is a matrix representing the weights of the first layer with $m$-neurons, and $w\in \mathbbm{R}^m$ represents the weights of the output layer. $\gamma=(\vect{\bld{A}},w)$ and the model is $f(x,\gamma)=w^T\bld{A}x$. The parameter space is the product space $\Gamma=\ca\times \cw$ of the first and second layers. \\
As we are working with compact spaces, minimization over the product space is equivalent to sequential minimization over the first and second layer, i.e.
\begin{align*}
    \min\limits_{\gamma\in\ca\times \cw} \cl(\gamma)= \min\limits_{w\in\cw}\min\limits_{\vect{\bld{A}}\in\ca} \cl(\gamma)
\end{align*}
Consequently, we can implement an alternating gradient descent (AGD) algorithm with step sizes $\alpha_1, \alpha_2$ to obtain the minimizer. The $t^{th}$ step of AGD involves the following two gradient descent steps:
\begin{align}
\label{eqn: AGD 2 layer w}
    \text{$w$-step:  }w^{(t+1)}&=w^{(t)}-\alpha_1\nabla_{w}\cl\left(w^{(t)},\iter{\bld{A}}{t}\right)\\
\label{eqn: AGD 2 layer A}
    \text{$\bld{A}$-step:  }\bld{A}^{(t+1)}&=\bld{A}^{(t)}-\alpha_2\nabla_{\bld{A}}\cl\left(w^{(t+1)},\iter{\bld{A}} {t}\right)
\end{align}
\paragraph{Identifiability issue} For the two-layer model, the optimal parameter $\gamma^*=(\vect{\bld{A}^*},w^*)$ isn't unique, however the corresponding logistic regression coefficient $\theta^*={\bld{A}^*}^Tw^*$ is unique and identifiable. The AGD steps in equations (\ref{eqn: AGD 2 layer w}, \ref{eqn: AGD 2 layer A}) induce a sequence in $\theta$ as well, with $\iter{\theta}{2t}={\iter{\bld{A}}{t}}^T\iter{w}{t}$ and $\iter{\theta}{2t+1}={\iter{\bld{A}}{t}}^T\iter{w}{t+1}$.  In the subsequent section, we can use this identification to tackle convergence rates of AGD in terms of $\theta$ and the loss $\cl(\theta)=\cl\left(w,\bld{A}\right)$. Furthermore, the notion of on/off manifold parameters can be extended in the two-layer settings in terms of $\theta=\bld{A}^Tw=(\on{\theta},\off{\theta})$.
\paragraph{Orthogonalization} For technical simplicity, we consider an orthogonalization step in addition to the $w$ and $\bld{A}$-steps. That is, before the $t^{th}$ iteration, $\iter{\bld{A}}{t}$ is column-orthogonalized such that ${\iter{\bld{A}}{t}}^T\iter{\bld{A}}{t}=\bld{I}_D$, and $\iter{w}{t}$ is recalibrated to preserve  $\iter{\theta}{t}$, i.e., ${\iter{\bld{A}}{t}}^T\iter{w}{t}$ keeps the same after orthogonalization. However, this assumption is practical, as orthogonality improves generalizability and curbs vanishing gradient issues. \citep{li2019orthogonal,achour2022existence}
\section{Main Results}
\subsection{Motivation}
\label{subsec: Motivation}
Consider the expected gradient and hessian of the loss w.r.t. the identifiable parameter $\theta$. Denote the score as $z=f(x,\gamma)$ and $\sigma(v)={(1+\exp(-v))}^{-1}$ as the standard sigmoid function.
\begin{align}
    \label{eqn: grad}
    \nabla_\theta\cl(\gamma)&=-\E\limits_{\sim x,y} yx \sigma(-y\cdot z)\\
    \nabla^2_{\theta}\cl(\gamma)&=\E\limits_{\sim x,y}xx^T\sigma(z)\sigma(-z)
    \label{eqn: hessian}
\end{align}
 \eqref{eqn: grad} is the gradient over the distribution over $x$. Thus, for $\on{x}$ belonging to the well-separated region, the gradients are accumulated constructively; in contrast, for all $\on{x}$ belonging to the overlapping region, gradients from each class cancel out and accumulate destructively. This implies that as the overlapping or $\nu$ increases, we expect weaker gradients for learning $\shortparallel$ direction.\\
Furthermore, \eqref{eqn: hessian} showcases the hessian/curvature of the loss w.r.t. $\theta$. Notice that the curvature is implicitly dependent on the term $\E xx^T$, which essentially captures the covariance structure of the data. The variance in the $\shortparallel$ direction is controlled by $\on{\sigma}^2$ inducing a larger curvature, compared to the variance in the $\perp$ direction inducing a small or flatter curvature. When implementing first-order gradient methods, the step size is bounded by the inverse of the largest curvature $\on{\sigma}^{-2}$; As we want to change the parameters carefully, if the loss is sensitive in certain directions. %As we want to take smaller steps if the curvature is large; 
However, this leads to slower learning in the flatter region, in this case, $\perp$ direction. Consequently, we expect faster convergence in the $\shortparallel$ direction and slower convergence in the $\perp$ direction, leading to a suboptimal solution with poor margins that is vulnerable to adversarial examples.\\ Technically, at a very high level, we bound the loss hessian based on variance matrices derived from the data structure. Subsequently, we use Taylor expansions of the loss, Lipschitz smoothness, strong convexity and PL-inequality-based arguments to derive parameter/loss convergence rates.\\
We formalize the prior intuitions in the following subsection with our main theorems.
\subsection{Theorems}
\label{subsec: theorems}
For both the logistic regression and two-layer linear network case, we denote the change in loss from $\theta^{t}\to \theta^{t+1}$ as: $\Delta \cl (t)=\cl(\iter{\theta}{t+1})-\cl(\iter{\theta}{t})$. Furthermore, the change in loss  contributed by the on/off manifold direction is denoted by $\on{\Delta} \cl (t)=\cl(\iter{\on{\theta}}{t+1},\iter{\off{\theta}}{t})-\cl(\iter{\theta}{t})$ and $\off{\Delta} \cl (t)=\cl(\iter{\on{\theta}}{t},\iter{\off{\theta}}{t+1})-\cl(\iter{\theta}{t})$  respectively.
\begin{theorem}[Progressive bounds] \label{thm: Progressive Bound}Given $(x,y)$ follows data distribution described in \subsecref{subsec: data dist}. For the $t^{th}$ iterate of $\theta$ induced by GD (\eqref{eqn: GD logistic}), $w$-step (\eqref{eqn: AGD 2 layer w}) or $\bld{A}$-step of AGD (\eqref{eqn: AGD 2 layer A}) we have:
\resizebox{0.48\textwidth}{!}{$
\begin{aligned}
\label{eqn: gradient on manifold}
    \nabla_{\on{\theta}}\cl(\iter{\theta}{t})&=-(1-\nu)\inner{\Vec{c_1}}{\onefunc_{d}\cdot\nicefrac{l-k}{2}}+\nu\inner{\Vec{c_2}}{\onefunc_d\cdot\nicefrac{k}{2}}\\
    \label{eqn: gradient off manifold}
    \nabla_{\off{\theta}}\cl(\iter{\theta}{t})&=-\pi(\inner{\vec{c_3}}{\off{\mu}^{(1)}})+(1-\pi)(\inner{(\onefunc_g-\vec{c_6})}{\off{\mu}^{(-1)}})\\
    &-\inner{(\vec{c_4}+\vec{c_5})}{\onefunc_g\cdot\nicefrac{\off{\sigma}}{4}}\nonumber
\end{aligned}$}
Furthermore, for appropriate choice of step sizes $\alpha,\alpha_1,\alpha_2\preceq\on{\sigma}^{-2}$ , we have:
\begin{align}
\label{eqn: progressive bnd directional}
    &\on{\Delta}\cl(t)\preceq-\snorm{\nabla_{\on{\theta}}\cl(\iter{\theta}{t})}^2\cdot\on{\sigma}^{-2};\\ &\off{\Delta}\cl(t)\preceq-\snorm{\nabla_{\off{\theta}}\cl(\iter{\theta}{t})}^2\cdot\on{\sigma}^{-2}\nonumber\\
    \label{eqn: progressive bnd}
    &\Delta\cl(t)\preceq-\snorm{\nabla_{\theta}\cl(\iter{\theta}{t})}^2\cdot\on{\sigma}^{-2}
\end{align}
Where $\Vec{c_i}$ are vectors dependent on $t$ with all their elements positive and $<1$; %$C_1, C_2, C_3$ are constants which dependent on the model choice.
\end{theorem}
 As gradient norms and the variances are positive, \eqref{eqn: progressive bnd directional} and \eqref{eqn: progressive bnd} imply that at each step of GD or AGD, the overall loss strictly decreases. In particular, the loss improves strictly in both $\perp,\shortparallel$ directions.
\paragraph{Effect of $\nu$} \eqref{eqn: gradient on manifold} decomposes the gradient in the on-manifold direction into two components corresponding to the well-separated and overlapping regions of the on-manifold distribution, respectively. Notice that the two terms are competing with each other, and as $\nu$ (overlapping coefficient) initially increases, $\snorm{\nabla_{\on{\theta}}\cl(\iter{\theta}{t})}$ tends to decrease due to cancellation. Hence, the loss improvement in $\shortparallel$ direction also diminishes (\eqref{eqn: progressive bnd directional}). With the extreme increase in $\nu$ even if $\snorm{\nabla_{\on{\theta}}\cl(\iter{\theta}{t})}$ is large, the classifier becomes agnostic of the original class direction, due to shift in the gradient direction favoring the overlapping component. %At each AGD, GD step, the progress made in the $\shortparallel$ direction decreases as $\nu$ or overlap increases.% Similarly, $\snorm{\nabla_{\off{\theta}}\cl(\iter{\theta}{t})}^2$ decreases as $\nicefrac{\off{\sigma}}{\on{\sigma}}$ decreases or the low-dimensionality of the distribution increases \eqref{eqn: gradient off manifold}. In turn, the loss improvement in the $\perp$ direction also diminishes with increasing low dimensionality \eqref{eqn: progressive bnd directional}.

\begin{theorem}[Parameter Convergence]
\label{thm: parameter convergence}
    Given $(x,y)$ follows data distribution described in \subsecref{subsec: data dist}. For both the logistic regression and two-layer linear network, let $T$ be the number of iterations w.r.t $\theta$ induced by GD (\eqref{eqn: GD logistic}), or $w$-step (\eqref{eqn: AGD 2 layer w}) and $\bld{A}$-step of AGD (\eqref{eqn: AGD 2 layer A}) with appropriate $\alpha,\alpha_1,\alpha_2\preceq\on{\sigma}^{-2}$; If:
    \begin{itemize}
        \item $\snorm{\iter{\on{\theta}}{T}-\on{\theta}^*}\leq \delta$ then, $T= \Omega(\log(\snorm{\iter{\theta}{0}-\on{\theta}^*}\cdot\delta^{-1}))$
        \item $\snorm{\iter{\off{\theta}}{T}-\off{\theta}^*}\leq \delta$ then, $T= \Omega((\nicefrac{\on{\sigma}}{\off{\sigma}})^2\cdot\log(\snorm{\iter{\theta}{0}-\off{\theta}^*}\cdot\delta^{-1}))$
    \end{itemize}
\end{theorem}
    
Theorem \ref{thm: parameter convergence} provides the convergence rate in terms of the identifiable parameter $\theta$ in both $\perp, \shortparallel$ directions. The convergence rate in the $\shortparallel$-direction is independent of the dimensionality, whereas for $\perp$-direction the rate depends on ${(\nicefrac{\off{\sigma}}{\on{\sigma}})}^{-1}$. Note that as $\nicefrac{\off{\sigma}}{\on{\sigma}}\to 0$ or $\off{\sigma}=0$ ($\off{x}$ follows a discrete distribution), the data distribution of $x$ becomes a $d$-dimensional manifold immersed in $D$-dimension space and the time required for convergence in $\perp$-direction blows to $\infty$.

\begin{theorem}[Loss Convergence]
\label{thm: loss convergence}
    Given $(x,y)$ follows data distribution described in \subsecref{subsec: data dist}. For both the logistic regression and two-layer linear network, let $T$ be the number of iterations w.r.t $\theta$ induced by GD (\eqref{eqn: GD logistic}), or $w$-step (\eqref{eqn: AGD 2 layer w}) and $\bld{A}$-step of AGD (\eqref{eqn: AGD 2 layer A}) with appropriate $\alpha,\alpha_1,\alpha_2\preceq\on{\sigma}^{-2}$. If $\theta^*$ is the optimal solution, such that $\left(\cl(\iter{\theta}{T})- \cl(\theta^*)\right)<\delta$; then:
    \begin{itemize}
        \item  $T=\min (r_1,r_2)$ if $\delta> C$.
        \item  $T= r_2$ if $\delta< C$.
    \end{itemize}
    where $r_1=\Omega(\log(|\cl(\iter{\theta}{0})-\cl(\theta^*)|\cdot(\delta-C)^{-1}))$, $r_2=\Omega((\nicefrac{\on{\sigma}}{\off{\sigma}})^2\cdot\log(|\cl(\iter{\theta}{0})-\cl(\theta^*)|\cdot\delta^{-1}))$ and $C=\Omega(\nu\log 2)$. 
\end{theorem}
Theorem \ref{thm: loss convergence} provides the convergence rates in terms of the loss. Additionally, it states that if the error tolerance $\delta >> C$, we can have fast convergence rate $r_1$ independent of dimensionality ($\nicefrac{\off{\sigma}^2}{\on{\sigma}^2}$). However, for an arbitrarily small $\delta<C$, the convergence rate $r_2$ can be significantly slower controlled by the dimensionality ($\nicefrac{\off{\sigma}}{\on{\sigma}}$) of the data. The threshold $C=\Omega(\nu\log 2)$ is essentially the minimum loss that can be attained by only training $\on{\theta}$ (\ref{proof: thm loss convergence}). As $\delta\to C$, the rate $r_2$ depended on dimensionality takes over. %Note that if we only use $\on{x}$ for classification. A perfect classifier will have $0$ loss corresponding to the non-overlapping region (Probability 1 for $\on{x}|_{y=1}$ and 0 for $\on{x}|_{y=-1}$) , and will have $\nicefrac{1}{2}$ probability for each class in the overlapping region. The probability of overlap is $\nu$. Hence, the lowest loss attained is $\nu\log 2$. 
\paragraph{Well separated on-manifold distribution}
Suppose there is no overlap, i.e., $\nu=0$, then Theorem \ref{thm: loss convergence} tells us that we can always have a fast convergence rate independent of dimensionality for any arbitrary error-tolerance $\delta$. The on-manifold coefficients are sufficient for perfect classification, corresponding to faster convergence. In this scenario, as long as convergence in $\shortparallel$ direction is attained, the data is perfectly classifiable. Hence, the classifier can achieve robustness just based on the on-manifold direction. (Fig \ref{fig:slow_convergence}b)
\paragraph{Illusion of convergence} When $\nu$ is small, the model can attain fast convergence to a small loss value; however, to perfectly classify the data distribution, convergence in both $\perp$ and $\shortparallel$ direction is still required as $\nu>0$ (Fig \ref{fig:slow_convergence}a). The model in this scenario will face adversarial vulnerability due to the poor convergence in the $\perp$ direction, even though the loss value is small.
\section{Towards second-order optimization}
\label{sec: second order}
Based on our motivation (Section \ref{subsec: Motivation}) and the proofs of the Theorems described in Section \ref{subsec: theorems}, the ill-conditioned nature of clean training is dictated by the usage of uniform small step size in both $\shortparallel$ and $\perp$ direction.
A small step size in the $\shortparallel$ direction is necessary due to large curvature, for careful descent. In the $\perp$ direction, where curvature is small, larger steps could be taken for efficiency. Variable step sizes for each direction will address ill-conditioning,
% Small step size in $\shortparallel$ direction makes sense as the curvature is large; hence, we must be more careful while implementing the descent algorithm. However, this isn't true for the $\perp$ direction where the curvature is very small/flat; we should take much larger steps in this direction for efficient steps or faster learning. A natural solution to the ill-conditioning is using variable step sizes for each direction. 
%If we choose appropriate adaptive step sizes for $\perp, \shortparallel$ direction ( Remark \ref{remark: variable step-size}), then. 
we can get convergence rates independent of dimensionality $\nicefrac{\off\sigma}{\on\sigma}$ for Theorems \ref{thm: parameter convergence},\ref{thm: loss convergence} in all cases. (Remark \ref{remark: variable step-size})
First-order methods can't automatically choose appropriate step sizes based on the curvature, but for the sake of argument, let us consider a second-order step like Newton's method instead of the GD step in \eqref{eqn: GD logistic}. \begin{align}
    \theta^{(t+1)}=\theta^{(t)}-\left(\nabla^2_{\theta}\cl(\theta^{(t)})\right)^{-1}\nabla_{\theta}\cl(\theta^{(t)}).
    \label{eqn: Newton logistic}
\end{align}   
The inverse hessian in the above equation is analogous to the uniform step size $\alpha$ in GD. However, for Newton's method, the effective step size is controlled by the inverse of the curvature. Large curvature or sharp directions are traversed carefully, and small curvature or flat directions are traversed liberally. The usage of second-order methods automatically induces variable step size, circumventing the ill-conditioning.
\paragraph{Hessian Estimate} In practice, computing the hessian inverse in \eqref{eqn: Newton logistic} is computationally expensive. Hence, instead of using exact Hessians, we can use preconditioning matrices that approximate the Hessian inverse well. 
% Quasi-Newton methods like BFGS and L-BFGS approximate the hessian using only gradient information. One can use the Fisher information matrix as a preconditioning, which is the basis for the natural gradient method. 
Our experiments use a relatively fast and inexpensive approximation to the hessian inverse known as the KFAC \citep{martens2015optimizing} preconditioner. KFAC scales efficiently in a distributed parallel setting for larger models. At its core, it is a natural gradient method that approximates the Fisher information as a layerwise block matrix and further approximates those blocks as being the Kronecker product of two much smaller matrices that are easier to compute.
\section{Experiments}
\label{sec: experiments}
Most computer vision datasets can be attributed to having a low-dimensional manifold structure \citep{pope2021intrinsic,osher2017low}. According to our framework, if there is an overlap in the manifold dimensions, the clean training is subjected to ill-conditioning, requiring considerable time to converge. Consequently, if a lack of convergence leads to adversarial vulnerability, the robust accuracy should increase with enough training.  Our discussions in Section \ref{sec: second order}, imply that with a second-order optimization scheme like KFAC, this robustness improvement should be much faster, and we could attain much more robust classifiers by just using clean training. We use the cross-entropy loss in our experiments, which is a multiclass generalization of the logistic loss we used in our theoretical setup.\\\\
We perform clean training on popular computer vision datasets MNIST \citep{lecun2010mnist} and FashionMNIST \citep{xiao2017/online} with a convolution neural network models (Tables \ref{tab:standard architecturel}, \ref{tab:cifar architecturel}
,\cite{lecun2015deep}) under two different optimization schemes first order and second order. We use the ADAM optimizer \citep{kingma2014adam}, which is considered one of the fastest first-order methods. We incorporate KFAC preconditioned matrices for our second-order optimization into the existing ADAM update. We use pytorch implementations \citep{pauloski2020kfac,pauloski2021kaisa} to compute the KFAC preconditioning.\\\\
At each training epoch, we subject the model to adversarial attacks to keep track of the model's robustness. We use $\ell_{\infty}$ Projected gradient descent (PGD)-attacks of strength $\epsilon$ to attack the models \citep{madry2017towards}. The robust accuracy is evaluated on the test data unbeknownst to training for various choices of attack strength $\epsilon$, where $\epsilon=0$ corresponds to the clean test accuracy. The total training is limited to 1000 epochs for illustrative purposes. We implement 10 runs for each model to get the avg and std dev. After $\sim10$ epochs, the clean training loss is $\sim0$ in all scenarios. (For additional details see appendix \ref{sec: exp details}, Code:\footnote{\url{https://anonymous.4open.science/r/Adv_Convergence_code-DF4B/}})

\begin{figure}[ht]
    \centering
    \includegraphics[width=0.45\textwidth]{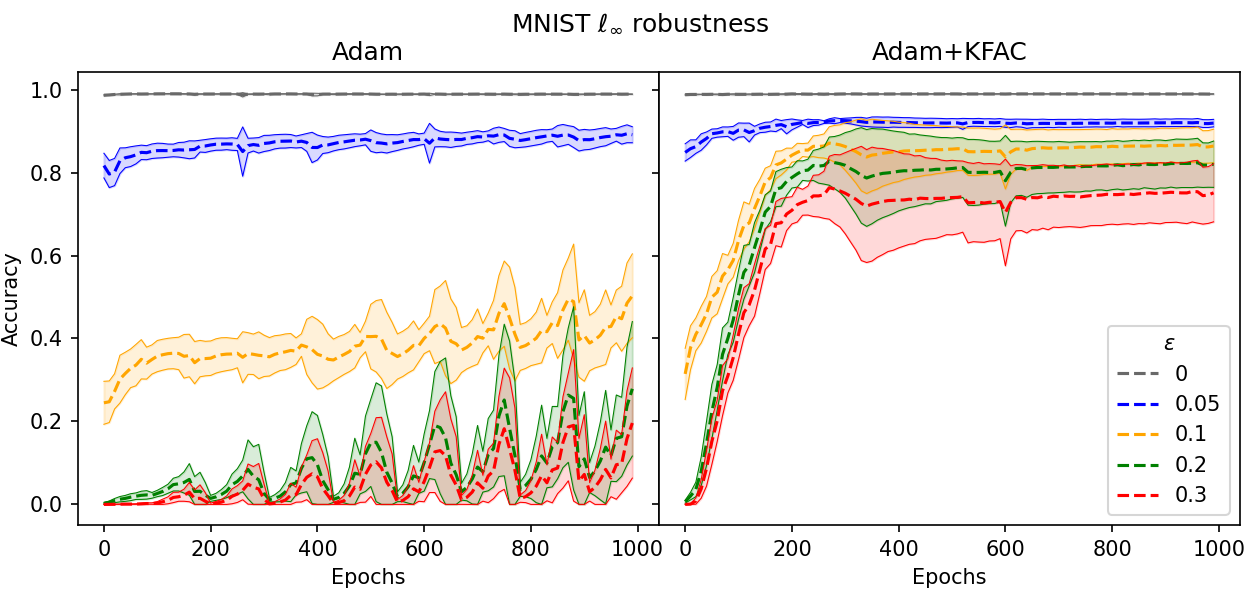}
    \includegraphics[width=0.45\textwidth]{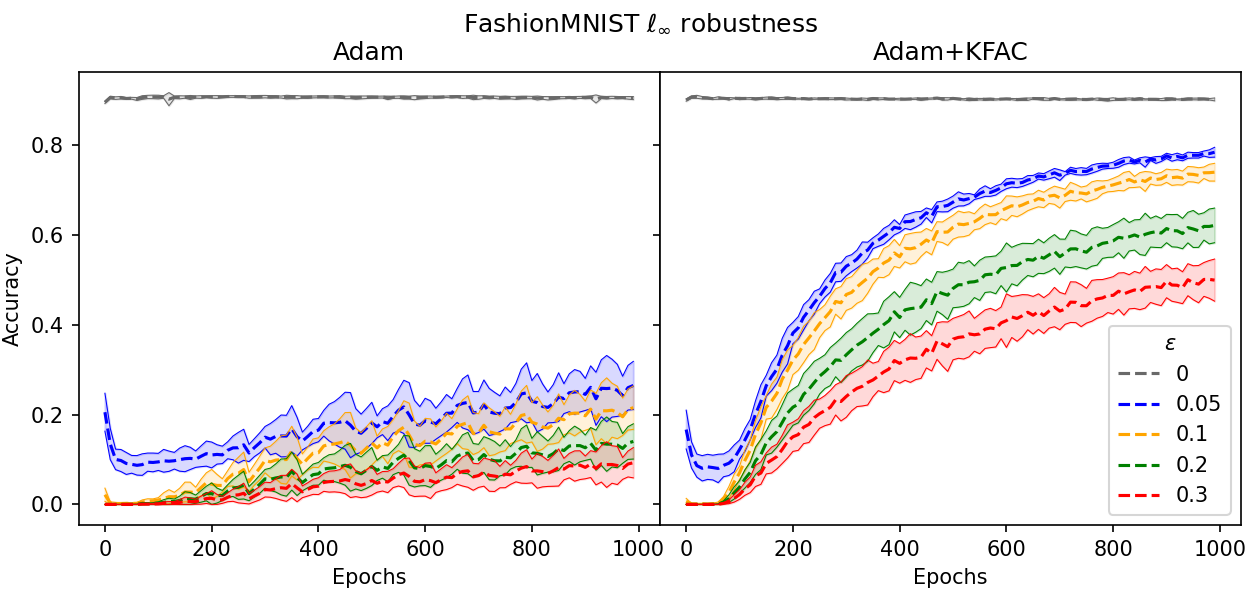}
    \caption{PGD $\ell_\infty$ robustness for clean MNIST (top)/FMNIST (bottom) model. Optimization Schemes (Left): First order ; (Right): Second Order.}
    \label{fig:mnist}
\end{figure}

%spiky but going up in general. For high epsilon is known to have almost 0 accuracy in clean training. But we see robustness increases in general, even though the clean loss is around 0 after 10 epochs and the clean test accuracy is unchanged. Suggesting the classifier is indeed improving with time, supporting our framework that , adversarial vulnerability is sourced from convergence issues.

% \begin{figure}[ht]
%     \centering
%     \includegraphics[width=0.45\textwidth]{paper_figs/fMNIST.png}
%     \caption{PGD -$\ell_\infty$ robustness for clean FashionMNIST model vs training epochs. Optimization Schemes (Left): First order  [ADAM]; (Right): Second Order optimization scheme [ADAM+KFAC].}
%     \label{fig:fmnist}
% \end{figure}

Figure \ref{fig:mnist} exhibits the results of our experiments described above for the MNIST and FMNIST datasets, respectively. It is evident that irrespective of the order of optimization, the robustness of the clean-trained model does increase with time under all attack strengths $(\epsilon)$ with time, as suggested by our theory. Even though the clean test accuracy is stagnant around $(\epsilon=0)$ is $\sim100,\sim 95$\%  (MNIST, FMNIST resp.) for the majority of the training, the adversarial robustness increases throughout. This validates our theory, suggesting a lack of convergence in $\perp$ direction leading to suboptimal classifiers that aren't large margin. Attaining, $\sim 0$ clean loss value and almost perfect test accuracy yet showcasing improvement in robustness throughout excessive training aligns with our discussions in section \ref{subsec: theorems} \emph{(Illusion of Convergence)} and the motivating illustration Fig \ref{fig: theoretical model} where the classifier is good enough on the data distribution, however, it hasn't attained convergence to optimal classifer. Additionally, the rate of robustness improvement for the second-order optimization is much faster, for the same amount of training epochs.
\begin{figure}[ht]
    \centering
    \begin{subfigure}{.23\textwidth}
  \centering
  \includegraphics[width=\linewidth]{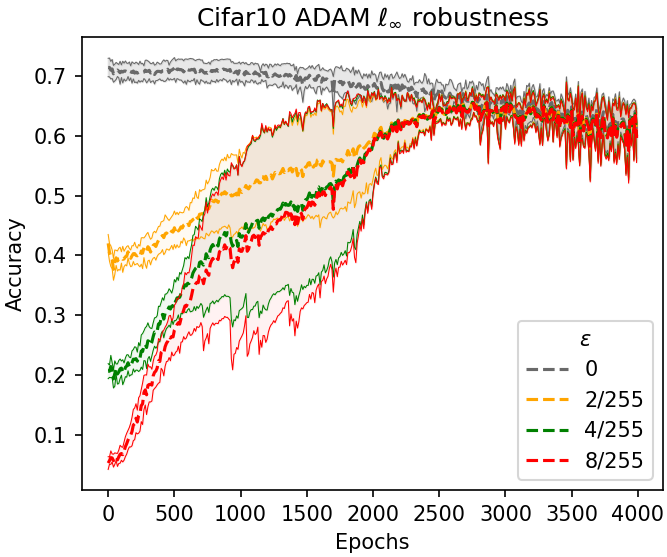}
  \caption{Standard}
  \label{fig:cifar_nobn}
\end{subfigure}%
\begin{subfigure}{.225\textwidth}
  \centering
  \includegraphics[width=\linewidth]{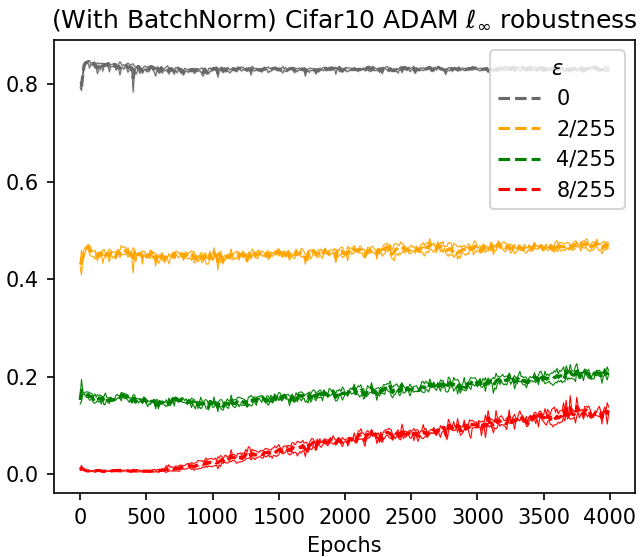}
  \caption{Batch Norm}
  \label{fig:cifar_bn}
\end{subfigure}
\caption{PGD $\ell_\infty$ robustness CIFAR10 clean model.}
    % \includegraphics[width=0.24\textwidth]{paper_figs/cifar10.png}
    % \includegraphics[width=0.24\textwidth]{paper_figs/cifar10_bn.png}
    % \caption{PGD $\ell_\infty$ robustness for CIFAR10 clean model.}
    \label{fig:cifar10}
\end{figure}

Additionally, we conduct a similar robustness experiment on the CIFAR10 dataset \citep{krizhevsky2009learning} with ADAM training (Figure \ref{fig:cifar_nobn}). KFAC is designed to scale up in a distributed setting for larger models. Although we couldn’t provide the KFAC version for CIFAR10 due to limited access to a single GPU, we conducted the first-order training for larger epochs ($4000$) to illustrate robustness improvement. We expect a second-order optimization scheme would yield similar results but with much smaller training epochs, as observed in the MNIST and FashionMNIST cases.
\paragraph{Unparalleled clean training performance}
 Attaining $\sim80$ and $\sim 40\%$ robust accuracy for $\epsilon=0.3$ (ADAM+KFAC in Fig \ref{fig:mnist}) in MNIST and FMNIST datasets respectively just by clean training is unprecedented. $\epsilon=0.3$ is a very large attack strength for these datasets, for context \cite{madry2017towards} reports a robust accuracy of only ~3.5\% for MNIST dataset undergoing clean training with the same attack strength. Similarly, for CIFAR10 Fig \ref{fig:cifar_nobn}, we attain $\sim 60\%$ robust accuracy for $\epsilon=8/255$ just using clean training. Traditional literature reports $0\%$ robust accuracy for clean-trained model and $47.04\%$ accuracy for the PGD-based adversarially trained model \citep{zhang2019theoretically}.
 \begin{figure}[!ht]
    \centering
    \includegraphics[width=0.49\textwidth]{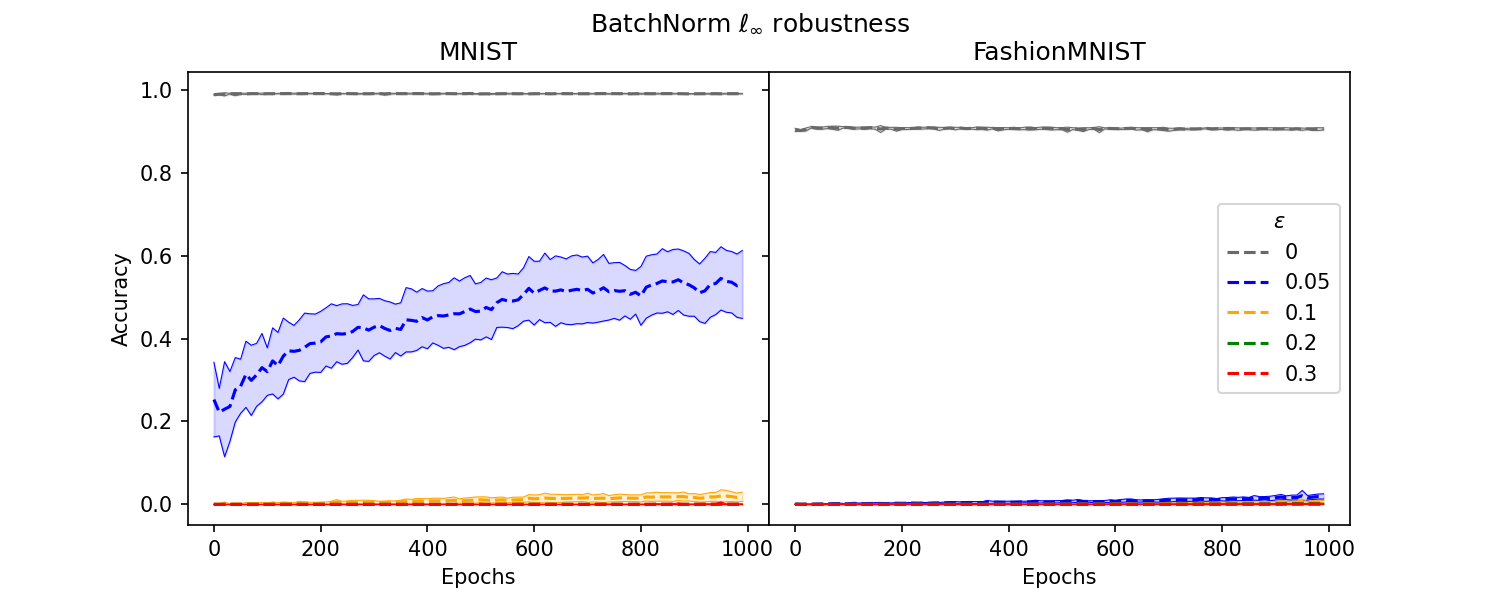}
    \caption{PGD -$\ell_\infty$ robustness for clean MNIST (left), FMNIST (right) models with batch-norm layers.}
    \label{fig:batchnorm}
\end{figure}
 \subsection{Vulnerability of Batch-Normalization}\label{sec:bn}
 Figures \ref{fig:batchnorm} and \ref{fig:cifar_bn} plot identical experiments but with the model architectures, including batch-normalization (Table \ref{tab:bn architecturel}, \ref{tab:cifar bn architecturel}). In this scenario, there is no robustness improvement with additional training, suggesting that batch-normalized models are vulnerable even at convergence. This aligns with the fact that contrary to traditional ReLU networks, which have an implicit bias towards maximum margin classifier, batch normalized networks have an implicit bias toward uniform margin \citep{cao2023implicit}. A uniform margin classifier doesn't enjoy the robustness properties of a maximum margin classifier due to its geometry. This raises robustness concerns for modern ML pipelines where batch normalization is a default.%\textcolor{red}{say implicit bias result in absence of redundant dimension holds for homogenous relu network and definition of maximum margin classifier by lu li }
\section{Discussion and Conclusion}
\label{sec: discussion}
We extend the notion of on and off-manifold dimensions to high and low-variance features. We explore a framework where inseparability in the on-manifold direction between classes causes adversarial vulnerability due to poor off-manifold learning from ill-conditioning. We present theoretical results supporting this hypothesis for a binary classification problem on a toy data distribution motavitaded by this concept. The theoretical analysis is done for the logistic regression case and 2-layer linear network for mathematical tractability. However, the idea of ill-conditioning facilitated by using first-order optimization for a model trained on a low-dimensional data manifold, resulting in adversarial vulnerability when data is inseparable in manifold direction, seems much more general. We verify this contrapositively by implementing experiments on MNIST, FMNIST, and CIFAR10 datasets with CNNs under a cross-entropy loss. Furthermore, we advocate using second-order methods that inherently circumvent ill-conditioning and lead us to a much more robust classifier just by using clean training. 
%\textcolor{red}{second order can only help with convergence but can't resolve vulnerability caused by redundant dimensions mentioned in refer works. Suggest using second order in combination with dimension reduction techniques.}
\paragraph{Adversarial Training}
In practice, adversarial training (AT)  is used to induce robustness in models \citep{goodfellow2014explaining,madry2017towards}. Clean training minimizes the loss for original data distribution, whereas AT minimizes the loss for the distribution of perturbed data or adversarial examples. AT requires a parameter $\epsilon$ corresponding to the attack strength or search radius of adversarial examples from the original data distribution. Given,  $\epsilon$ is less than the \emph{margin} between the distinct classes, AT is doing clean training on the adversarial examples distribution. This forms an $\epsilon$-ball cover around the original data distribution. Even if the original data distribution is low dimensional, the adversarial example distribution will be the exact ambient dimension. In our framework, this translates to the fact that the off-manifold or low variance features will also have an $ \epsilon $ -ball cover, effectively increasing its variance.
% For the sake of understanding, consider the extreme case where the off-manifold feature is discrete or has 0 variance; however, in the adversarial example distribution, the variance is non-zero due to the continuous epsilon ball around it.
Hence, $\nicefrac{\off{\sigma}}{\on{\sigma}}$ increases from the original data distribution to the adversarial counterpart, thereby improving the convergence (Theorems \ref{thm: parameter convergence}, \ref{thm: loss convergence}). This explains why adversarial training can find a robust classifier while clean training cannot, even though a robust classifier is also an optimal solution to the clean loss.
% \subsection*{Broader Impacts, Limitation, and Future Direction}

% In this work, we introduce the concept of manifold inseparability as a source of adversarial vulnerability. A better theoretical understanding of the causes of adversarial vulnerability will help the scientific community build deep learning models that are inherently robust against attackers, leaping toward trustworthy AI technology.

% Our theoretical results, derived from uniformly distributed data on logistic and 2-layer linear models, provide a mathematically tractable foundation. We aim to generalize these findings to 2-layer ReLU models and broader distributions in the future. Notably, second-order methods address convergence issues contributing to adversarial vulnerability. However, it cannot address the vulnerabilities caused by redundant data dimensions systemic to implicit bias \citep{pmlr-v238-haldar24a,melamed2024adversarial}. We encourage the community to combine second-order optimization with dimension-reduction techniques for improved robustness.
\newpage     
\bibliographystyle{apalike}
\bibliography{reference}

\begin{thebibliography}{}

\bibitem[Achour et~al., 2022]{achour2022existence}
Achour, E.~M., Malgouyres, F., and Mamalet, F. (2022).
\newblock Existence, stability and scalability of orthogonal convolutional
  neural networks.
\newblock {\em Journal of Machine Learning Research}, 23(347):1--56.

\bibitem[Alemany and Pissinou, 2020]{alemany2020dilemma}
Alemany, S. and Pissinou, N. (2020).
\newblock The dilemma between data transformations and adversarial robustness
  for time series application systems.
\newblock {\em arXiv preprint arXiv:2006.10885}.

\bibitem[Aparne et~al., 2022]{aparne2022pca}
Aparne, G., Banburski, A., and Poggio, T. (2022).
\newblock Pca as a defense against some adversaries.
\newblock Technical report, Center for Brains, Minds and Machines (CBMM).

\bibitem[Boyd and Vandenberghe, 2004]{boyd2004convex}
Boyd, S.~P. and Vandenberghe, L. (2004).
\newblock {\em Convex optimization}.
\newblock Cambridge university press.

\bibitem[Cao et~al., 2023]{cao2023implicit}
Cao, Y., Zou, D., Li, Y., and Gu, Q. (2023).
\newblock The implicit bias of batch normalization in linear models and
  two-layer linear convolutional neural networks.

\bibitem[Carlini and Wagner, 2017]{carlini2017towards}
Carlini, N. and Wagner, D. (2017).
\newblock Towards evaluating the robustness of neural networks.
\newblock In {\em 2017 ieee symposium on security and privacy (sp)}, pages
  39--57. Ieee.

\bibitem[Chen et~al., 2021]{chen2021equivalence}
Chen, Y., Huang, W., Nguyen, L., and Weng, T.-W. (2021).
\newblock On the equivalence between neural network and support vector machine.
\newblock {\em Advances in Neural Information Processing Systems},
  34:23478--23490.

\bibitem[Domingos, 2020]{domingos2020every}
Domingos, P. (2020).
\newblock Every model learned by gradient descent is approximately a kernel
  machine.
\newblock {\em arXiv preprint arXiv:2012.00152}.

\bibitem[Frei et~al., 2024]{frei2024double}
Frei, S., Vardi, G., Bartlett, P., and Srebro, N. (2024).
\newblock The double-edged sword of implicit bias: Generalization vs.
  robustness in relu networks.
\newblock {\em Advances in Neural Information Processing Systems}, 36.

\bibitem[Goodfellow et~al., 2016]{Goodfellow-et-al-2016}
Goodfellow, I., Bengio, Y., and Courville, A. (2016).
\newblock {\em Deep Learning}.
\newblock MIT Press.
\newblock \url{http://www.deeplearningbook.org}.

\bibitem[Goodfellow et~al., 2014]{goodfellow2014explaining}
Goodfellow, I.~J., Shlens, J., and Szegedy, C. (2014).
\newblock Explaining and harnessing adversarial examples.
\newblock {\em arXiv preprint arXiv:1412.6572}.

\bibitem[Haldar et~al., 2024]{pmlr-v238-haldar24a}
Haldar, R., Xing, Y., and Song, Q. (2024).
\newblock Effect of ambient-intrinsic dimension gap on adversarial
  vulnerability.
\newblock In Dasgupta, S., Mandt, S., and Li, Y., editors, {\em Proceedings of
  The 27th International Conference on Artificial Intelligence and Statistics},
  volume 238 of {\em Proceedings of Machine Learning Research}, pages
  1090--1098. PMLR.

\bibitem[Jha et~al., 2018]{jha2018detecting}
Jha, S., Jang, U., Jha, S., and Jalaian, B. (2018).
\newblock Detecting adversarial examples using data manifolds.
\newblock In {\em MILCOM 2018-2018 IEEE Military Communications Conference
  (MILCOM)}, pages 547--552. IEEE.

\bibitem[Kingma and Ba, 2014]{kingma2014adam}
Kingma, D.~P. and Ba, J. (2014).
\newblock Adam: A method for stochastic optimization.
\newblock {\em arXiv preprint arXiv:1412.6980}.

\bibitem[Krizhevsky et~al., 2009]{krizhevsky2009learning}
Krizhevsky, A. et~al. (2009).
\newblock Learning multiple layers of features from tiny images.

\bibitem[LeCun et~al., 2015]{lecun2015deep}
LeCun, Y., Bengio, Y., and Hinton, G. (2015).
\newblock Deep learning.
\newblock {\em nature}, 521(7553):436--444.

\bibitem[LeCun et~al., 2010]{lecun2010mnist}
LeCun, Y., Cortes, C., and Burges, C. (2010).
\newblock Mnist handwritten digit database.
\newblock {\em ATT Labs [Online]. Available: http://yann.lecun.com/exdb/mnist},
  2.

\bibitem[Li et~al., 2019]{li2019orthogonal}
Li, S., Jia, K., Wen, Y., Liu, T., and Tao, D. (2019).
\newblock Orthogonal deep neural networks.
\newblock {\em IEEE transactions on pattern analysis and machine intelligence},
  43(4):1352--1368.

\bibitem[Lin et~al., 2020]{lin2020dual}
Lin, W.-A., Lau, C.~P., Levine, A., Chellappa, R., and Feizi, S. (2020).
\newblock Dual manifold adversarial robustness: Defense against lp and non-lp
  adversarial attacks.
\newblock {\em Advances in Neural Information Processing Systems},
  33:3487--3498.

\bibitem[Lindqvist et~al., 2018]{lindqvist2018autogan}
Lindqvist, B., Sugrim, S., and Izmailov, R. (2018).
\newblock Autogan: Robust classifier against adversarial attacks.
\newblock {\em arXiv preprint arXiv:1812.03405}.

\bibitem[Liu et~al., 2020]{liu2020understanding}
Liu, J., Jiang, G., Bai, Y., Chen, T., and Wang, H. (2020).
\newblock Understanding why neural networks generalize well through gsnr of
  parameters.
\newblock {\em arXiv preprint arXiv:2001.07384}.

\bibitem[Lyu and Li, 2019]{lyu2019gradient}
Lyu, K. and Li, J. (2019).
\newblock Gradient descent maximizes the margin of homogeneous neural networks.
\newblock {\em arXiv preprint arXiv:1906.05890}.

\bibitem[Lyu et~al., 2021]{lyu2021gradient}
Lyu, K., Li, Z., Wang, R., and Arora, S. (2021).
\newblock Gradient descent on two-layer nets: Margin maximization and
  simplicity bias.

\bibitem[Madry et~al., 2017]{madry2017towards}
Madry, A., Makelov, A., Schmidt, L., Tsipras, D., and Vladu, A. (2017).
\newblock Towards deep learning models resistant to adversarial attacks.
\newblock {\em arXiv preprint arXiv:1706.06083}.

\bibitem[Martens and Grosse, 2015]{martens2015optimizing}
Martens, J. and Grosse, R. (2015).
\newblock Optimizing neural networks with kronecker-factored approximate
  curvature.
\newblock In {\em International conference on machine learning}, pages
  2408--2417. PMLR.

\bibitem[Melamed et~al., 2024]{melamed2024adversarial}
Melamed, O., Yehudai, G., and Vardi, G. (2024).
\newblock Adversarial examples exist in two-layer relu networks for low
  dimensional linear subspaces.
\newblock {\em Advances in Neural Information Processing Systems}, 36.

\bibitem[Min and Vidal, 2024]{min2024can}
Min, H. and Vidal, R. (2024).
\newblock Can implicit bias imply adversarial robustness?
\newblock {\em arXiv preprint arXiv:2405.15942}.

\bibitem[Nacson et~al., 2022]{nacson2022stochastic}
Nacson, M.~S., Srebro, N., and Soudry, D. (2022).
\newblock Stochastic gradient descent on separable data: Exact convergence with
  a fixed learning rate.

\bibitem[Nesterov, 2013]{nesterov2013introductory}
Nesterov, Y. (2013).
\newblock {\em Introductory lectures on convex optimization: A basic course},
  volume~87.
\newblock Springer Science \& Business Media.

\bibitem[Osher et~al., 2017]{osher2017low}
Osher, S., Shi, Z., and Zhu, W. (2017).
\newblock Low dimensional manifold model for image processing.
\newblock {\em SIAM Journal on Imaging Sciences}, 10(4):1669--1690.

\bibitem[Pauloski et~al., 2021]{pauloski2021kaisa}
Pauloski, J.~G., Huang, Q., Huang, L., Venkataraman, S., Chard, K., Foster, I.,
  and Zhang, Z. (2021).
\newblock Kaisa: {A}n {A}daptive {S}econd-{O}rder {O}ptimizer {F}ramework for
  {D}eep {N}eural {N}etworks.
\newblock In {\em Proceedings of the International Conference for High
  Performance Computing, Networking, Storage and Analysis}, SC '21, New York,
  NY, USA. Association for Computing Machinery.

\bibitem[Pauloski et~al., 2020]{pauloski2020kfac}
Pauloski, J.~G., Zhang, Z., Huang, L., Xu, W., and Foster, I.~T. (2020).
\newblock Convolutional {N}eural {N}etwork {T}raining with {D}istributed
  {K}-{FAC}.
\newblock In {\em Proceedings of the International Conference for High
  Performance Computing, Networking, Storage and Analysis}, SC '20. IEEE Press.

\bibitem[Pope et~al., 2021]{pope2021intrinsic}
Pope, P., Zhu, C., Abdelkader, A., Goldblum, M., and Goldstein, T. (2021).
\newblock The intrinsic dimension of images and its impact on learning.
\newblock {\em arXiv preprint arXiv:2104.08894}.

\bibitem[Rosset et~al., 2003]{rosset2003margin}
Rosset, S., Zhu, J., and Hastie, T. (2003).
\newblock Margin maximizing loss functions.
\newblock {\em Advances in neural information processing systems}, 16.

\bibitem[Shamir et~al., 2021]{shamir2021dimpled}
Shamir, A., Melamed, O., and BenShmuel, O. (2021).
\newblock The dimpled manifold model of adversarial examples in machine
  learning.
\newblock {\em arXiv preprint arXiv:2106.10151}.

\bibitem[Stutz et~al., 2019]{stutz2019disentangling}
Stutz, D., Hein, M., and Schiele, B. (2019).
\newblock Disentangling adversarial robustness and generalization.
\newblock In {\em Proceedings of the IEEE/CVF Conference on Computer Vision and
  Pattern Recognition}, pages 6976--6987.

\bibitem[Szegedy et~al., 2013]{szegedy2013intriguing}
Szegedy, C., Zaremba, W., Sutskever, I., Bruna, J., Erhan, D., Goodfellow, I.,
  and Fergus, R. (2013).
\newblock Intriguing properties of neural networks.
\newblock {\em arXiv preprint arXiv:1312.6199}.

\bibitem[Wei et~al., 2019]{wei2019regularization}
Wei, C., Lee, J.~D., Liu, Q., and Ma, T. (2019).
\newblock Regularization matters: Generalization and optimization of neural
  nets vs their induced kernel.
\newblock {\em Advances in Neural Information Processing Systems}, 32.

\bibitem[Xiao et~al., 2017]{xiao2017/online}
Xiao, H., Rasul, K., and Vollgraf, R. (2017).
\newblock Fashion-mnist: a novel image dataset for benchmarking machine
  learning algorithms.

\bibitem[Xiao et~al., 2022]{xiao2022understanding}
Xiao, J., Yang, L., Fan, Y., Wang, J., and Luo, Z.-Q. (2022).
\newblock Understanding adversarial robustness against on-manifold adversarial
  examples.
\newblock {\em arXiv preprint arXiv:2210.00430}.

\bibitem[Zhang et~al., 2019]{zhang2019theoretically}
Zhang, H., Yu, Y., Jiao, J., Xing, E., El~Ghaoui, L., and Jordan, M. (2019).
\newblock Theoretically principled trade-off between robustness and accuracy.
\newblock In {\em International conference on machine learning}, pages
  7472--7482. PMLR.

\bibitem[Zhang et~al., 2022]{zhang2022manifold}
Zhang, W., Zhang, Y., Hu, X., Goswami, M., Chen, C., and Metaxas, D.~N. (2022).
\newblock A manifold view of adversarial risk.
\newblock In {\em International Conference on Artificial Intelligence and
  Statistics}, pages 11598--11614. PMLR.

\end{thebibliography}

\appendix
\onecolumn
\section{Proofs}
In all our proofs, we assume that $Rank(\bld{A})=D$, as when $Rank(\bld{A})<D$, we can just work with a block of $\bld{A}$ that is full rank. The gradient of $w$ corresponding to the co-kernel of $\bld{A}$ will always be zero, and it doesn't affect any of the gradient descent steps pertaining to our analysis.
\subsection{Reparametrization}
\begin{lemma}[Reparametrization] For the $t^{th}$ iterate of w-step \eqref{eqn: AGD 2 layer w} and $\bld{A}$-step of AGD  \eqref{eqn: AGD 2 layer A} there exist reparametrizations $\iter{\wt}{t}$ and $\iter{\At}{t}$ such that
$\iter{\theta}{2t+1}=(\on{\iter{\theta}{2t+1}},\off{\iter{\theta}{2t+1}})=(\on{\iter{\at}{t}},\off{\iter{\at}{t}})$ and $\iter{\theta}{2t}=(\on{\iter{\theta}{2t}},\off{\iter{\theta}{2t}})=(\on{\iter{\wt}{t}},\off{\iter{\wt}{t}})$. Where $\wt=(\on{\wt},\off{\wt},\flt{\wt}), \vect{\At}=(\on{\at},\off{\at},\flt{\at})$ and $\nabla\cl_{\flt{\wt}}=\bld{0}_{m-D},\nabla\cl_{\flt{\at}}=\bld{0}_{mD-D}$.
\label{lem: reparametrize}
\end{lemma}
\begin{proof}
Recall that, for the two-layer linear network, $\gamma=(\vect{\bld{A}},w)$ and $\theta=\bld{A}^Tw$, with $\nabla_\theta\cl(\gamma)=-\E\limits_{\sim x,y} yx \sigma(-y\cdot z)$ (\eqref{eqn: grad}). Our reparametrization/change of basis will help us separate the components of the first and second layer into $\shortparallel,\perp,\vdash$ directions corresponding to the identifiable parameter $\theta=(\on{\theta},\off{\theta})$. $\vdash$ components correspond to redundant directions that don't undergo any change in the AGD step. 
\subsubsection{w-step}
\label{subsec: w-step}
In this step of AGD, the second layer is fixed to $\iter{\bld{A}}{t}$. Given $\bld{A}=\iter{\bld{A}}{t}$, the gradient w.r.t. the first layer parameter $w$ is:
\begin{equation}
    \nabla_w\cl(\iter{\theta}{2t})=-\bE\sigma(-y\cdot z)y\bld{A}x
\end{equation}
$\bld{A}$ is a rectangular matrix ($m\times D; m\geq D$). $Rank(\bld{A})\leq D$. Hence there exists a change of basis $\bld{B}\in \mathbbm{R}^{m\times m}$ such that $\bld{A}=\bld{B}[\bld{e}_1\dots \bld{e}_D]$ where $\bld{e}_i$ is the standard basis in $\mathbbm{R}^m$ with $i^{th}$ coordinate being 1, and rest 0. This choice of standard basis for the $Im(\bld{A})$ loses no generality as we can always have $\Tilde{\bld{B}}=\bld{B}\bld{P}$ which permutes the basis in such a way so that the transformation adheres to the choice of a subset of the standard basis. In general $\bld{B}=\begin{bmatrix}
   \bld{A} & \bld{C} 
\end{bmatrix}$ has to take this kind of form, for any choice of $\bld{C}\in \mathbbm{R}^{m\times m-D}$. Hence, the gradient w.r.t.can be expressed as:
\begin{equation}
    \nabla_w\cL(\iter{\theta}{2t})=-\bE\sigma(-y\cdot z)y\bld{B}\begin{pmatrix}
        x\\
        \bld{0}_{m-D}
    \end{pmatrix}
\end{equation}
We can choose an appropriate full rank $\bld{C}$ outside the span of $\bld{B}$ to get an inverse transformation $\bld{B}^{-1}$. Note in the case that $Rank(\bld{A})<D$ we can just work with block inverse of the full rank part in $\bld{B}$; as the gradient of $w$ corresponding to the co-kernel of $\bld{A}$ will always be zero, so it isn't of interest.\\
Consider the reparametrization $\iter{\wt}{t}=\bld{B}^{T}\iter{w}{t}$, with gradients:
\begin{equation}
 \nabla_{\Tilde{w}}\cL(\iter{\theta}{2t})=\bld{B}^{-1}\nabla_w\cL=-\bE\sigma(-y\cdot z)y\bld{A}x=-\bE\sigma(-y\cdot z)y\begin{pmatrix}
        x\\
        \bld{0}_{m-D}
    \end{pmatrix}  
    \label{eqn: reparemtrize w}
\end{equation}
Note that, $\iter{\theta}{2t}={\iter{\bld{A}}{t}}^T\iter{w}{t}={(\bld{B}[\bld{e}_1\dots \bld{e}_D])}^T\iter{w}{t}=[\bld{e}_1\dots \bld{e}_D]^T\iter{\wt}{t}$. Therefore, it follows that  there exists $\on{\wt},\off{\wt}$ components corresponding to $\on{\theta},\off{\theta}$ at every $t$. Note that the remaining components $\flt{\wt}$ are redundant and have zero gradients. 
\begin{align}
 \nabla_{\Tilde{w}}\cL(\theta)&=-\bE\sigma(-y\cdot z)y\begin{pmatrix}
        x\\
        \bld{0}_{m-D}
    \end{pmatrix}   \\
    \nabla_{\on{\wt}}\cL(\theta)&=-\bE\sigma(-y\cdot z)y\on{x}\\
    \nabla_{\off{\wt}}\cL(\theta)&=-\bE\sigma(-y\cdot z)y\off{x}\\
    \nabla_{\flt{\wt}}\cL(\theta)&=\bld{0}_{m-D}
\end{align}
\subsubsection{A-step}
\label{subsec: A-step}
In this step, the first layer is fixed to $\iter{w}{t+1}$. Given $w=\iter{w}{t+1}$, the gradient w.r.t. the second layer parameters $\bld{A}$ is :
\begin{align}
    \nabla_{\bld{A}}\cL&=-\bE\sigma(-y\cdot z)ywx^T\\
    \nabla_{\vect\bld{A}}\cL&=-\bE\sigma(-y\cdot z)yw\otimes x
\end{align}
Consider the reparametrization $\At$ such that $\iter{\bld{A}}{t}=\bU^T\iter{\At}{t}=\begin{pmatrix}
    u_1\\
    \vdots\\
    u_m
\end{pmatrix}^T\At$.  The columns are chosen using Gram-Schmidt such that, $u_1=\nicefrac{w}{\snorm{w}^{2}}$ and $\Inner{u_j}{w}=0,\snorm{u_j}=1$ for all $j\neq 1$. The gradient in the re-parametrized space is then:
\begin{align}
    \nabla_{\At}\cL&=-\bE\sigma(-y\cdot z)y\bU wx^T=-\bE\sigma(-y\cdot z)y\bld{e}_1x^T \label{eqn: reparemtrize A}\\
    \nabla_{\vect\At}\cL&=-\bE\sigma(-y\cdot z)y\bld{e}_1\otimes x
\end{align}
Note that, \begin{equation*}
    \iter{\theta}{2t+1}=\bld{A}^{(t)^T}\iter{w}{t+1}={({\bU}^T\iter{\At}{t})}^Tw^{t+1}=\At^{(t)^T}\bU w^{t+1}=\At^{(t)^T}\bld{e}_1
\end{equation*}. Therefore, it follows that there exists $\on{a},\off{a}$ components for $\vect\At$ corresponding to $\on{\theta},\off{\theta}$ at every $t$. Hence, $\vect\At=(\on{a},\off{a},\flt{a})$ with: 
\begin{align}
    \nabla_{\on{a}}\cL(\theta)&=-\bE\sigma(-y\cdot z)y\on{x}\\
    \nabla_{\off{a}}\cL(\theta)&=-\bE\sigma(-y\cdot z)y\off{x}\\
    \nabla_{\flt{a}}\cL(\theta)&=\bld{0}_{mD-D}
\end{align}
\end{proof}

\subsection{Hessian}
\begin{lemma}[Lipschitz smoothness and Strong Convexity] 
\label{lem:lipschitz/strong convex}For any value of $\gamma$ in both logistic and two-layer-linear setting. The loss Hessians w.r.t. identifiable parameter $\theta$ can be bounded as follows:
\begin{align}
     c_l\cdot\begin{pmatrix}
    \on\sigma^2\bld{I}_d &\bld{0}\\
    \bld{0} & \off\sigma^2\bld{I}_g
\end{pmatrix}&<\nabla^2_{\theta}\cl(\gamma)<c_u\cdot\begin{pmatrix}
    \on\sigma^2\bld{I}_d &\bld{0}\\
    \bld{0} & \off\sigma^2\bld{I}_g\end{pmatrix}\label{eqn: hessian bounds}\\
c_l\off\sigma^2\bld{I}_g&<\nabla^2_{\off{\theta}}\cl(\gamma)<c_u\off{\sigma^2}\bld{I}_g\label{eqn: hessian bound off-manifold}\\
c_l\on\sigma^2\bld{I}_d&<\nabla^2_{\on{\theta}}\cl(\gamma)<c_u\on{\sigma^2}\bld{I}_d\label{eqn: hessian bound on-manifold}
\end{align}
Here $c_l,c_u>0$ are constants. Consequently, $\cl(\gamma)$ is convex in $\theta,\on{\theta},\off{\theta}$.
\end{lemma}
\begin{proof}
    The Hessian of the loss w.r.t. the identifiable parameter $\theta$ is 
$\nabla^2_{\theta}\cl(\gamma)=\E\limits_{\sim x,y}xx^T\sigma(z)\sigma(-z)$ [\eqref{eqn: hessian}]. As our analysis is restricted to the compact space, $\sigma(z)$ is strictly between 0 and 1, as $\sigma(z)=1/0$ if and only if $\theta=\pm \infty$. Let $c_l>0$ be the lower bound of the entropy $\sigma(z)\sigma(-z)$. Then, using positive definiteness and integrating both sides (Here $>,<$ between matrices corresponds to the difference being positive/negative definite):
\begin{align*}
    \E\limits_{\sim x,y}xx^Tc_l&<\nabla^2_{\theta}\cl(\gamma)<\E\limits_{\sim x,y}xx^T\nonumber\\
    (\Sigma_x+\E\limits_{\sim x,y}x{\E\limits_{\sim x,y}x}^T)\cdot c_l&<\nabla^2_{\theta}\cl(\gamma)<\Sigma_x+\E\limits_{\sim x,y}x{\E\limits_{\sim x,y}x}^T
\end{align*}
Where $\Sigma_x=\begin{pmatrix}
    \on\sigma^2\bld{I}_d &\bld{0}\\
    \bld{0} & \off\sigma^2\bld{I}_g
\end{pmatrix}$ is the covariance matrix. The vector outer-product $\E x \E x^T$ has eigenvalues $\snorm{\mu}^2=\norm{\E x}^2,0,0\dots, 0$. Hence, $\bld{0}_D\leq \E x\E x^T\leq \snorm{\mu}^2\bld{I}_D$. This implies that:
\begin{align}
    \Sigma_x\cdot c_l &<\nabla^2_{\theta}\cl(\gamma)<\Sigma_x+\snorm{\mu}^2\bld{I}_D\nonumber\\
    %\intertext{As $\on{\mu}^{(y)},\off{\mu}^{(y)}\asymp \on{\sigma}$,\, $\norm{\E x}^2\asymp \on{\sigma}^2$. Also $\off{\sigma}=c_0\on{\sigma}$, where $0\leq c_0<1$.}
    \intertext{For $0<\snorm{\mu},\on{\sigma},\off{\sigma}<\infty$. We have $\snorm{\mu}=\on{c}\on\sigma=\off{c}\off\sigma$}
    c_l\cdot\begin{pmatrix}
    \on\sigma^2\bld{I}_d &\bld{0}\\
    \bld{0} & \off\sigma^2\bld{I}_g
\end{pmatrix}&<\nabla^2_{\theta}\cl(\gamma)<\begin{pmatrix}
    \on\sigma^2\bld{I}_d &\bld{0}\\
    \bld{0} &\off\sigma^2\bld{I}_g
\end{pmatrix}+\begin{pmatrix}
    \on{c}^2\on\sigma^2\bld{I}_d &\bld{0}\\
    \bld{0} &\off{c}^2\off\sigma^2\bld{I}_g
\end{pmatrix}\nonumber\\
\intertext{Then there exists constants $c_u=\max\{1+\on{c}^2,1+\off{c}^2\}$ such that}
c_l\cdot\begin{pmatrix}
    \on\sigma^2\bld{I}_d &\bld{0}\\
    \bld{0} & \off\sigma^2\bld{I}_g
\end{pmatrix}&<\nabla^2_{\theta}\cl(\gamma)<c_u\cdot\begin{pmatrix}
    \on\sigma^2\bld{I}_d &\bld{0}\\
    \bld{0} & \off\sigma^2\bld{I}_g
\end{pmatrix}\nonumber
\end{align}
\begin{equation*}
    c_l\off\sigma^2\bld{I}_g<\nabla^2_{\off{\theta}}\cl(\gamma)<c_u\off{\sigma^2}\bld{I}_g
\end{equation*}
\begin{equation*}
    c_l\on\sigma^2\bld{I}_d<\nabla^2_{\on{\theta}}\cl(\gamma)<c_u\on{\sigma^2}\bld{I}_d
\end{equation*}
As the above Hessians are positive definite, convexity follows.
\end{proof}

\begin{proof}[Proof of Theorem \ref{thm: Progressive Bound}]
\subsection{Gradient Decomposition}

%\textbf{\large Gradient Decomposition}\\
From \eqref{eqn: grad} $\nabla_\theta\cl(\gamma)=-\E\limits_{\sim x,y} yx \sigma(-y\cdot z)$. We will tackle the $\shortparallel/\perp$ components of the gradient separately.
\subsubsection{On Manifold}
Consider $\nabla_{\on{\theta}}\cl(\gamma)=-\E\limits_{\sim x,y} y\on{x} \sigma(-y\cdot z)$. 
 $\nabla_{\on{\theta}}\cl$ can be decomposed into $\nabla_{\on{\theta}}\cl^{\neg\emptyset},\nabla_{\on{\theta}}\cl^{\emptyset}$, the gradients arising from overlapping and non-overlapping distribution  of $\on{x}$, with probabilities $\nu,1-\nu$ respectively. The distribution of $x$ corresponding to the non-overlapping part and overlapping part is  $\on{x}|_{y,\emptyset}\sim \mathcal{U}[0,(l-k)\cdot y\cdot\onefunc_d]$ and $\on{x}|_{y,\neg\emptyset}\sim \mathcal{U}[-k.y\cdot \onefunc_d,0)$ resp. One can validate this with the definition of $\nu$ (Section \ref{sec:ovl}),  $\on{x}|_{y,\emptyset}$ contributes $0$ in the computation of $\nu$, while $\on{x}|_{y,\neg\emptyset}$ has non-zero contribution over the support.
 \begin{align}
     \nabla_{\on{\theta}}\cl(\gamma)=-\E\limits_{\sim x,y} y\on{x} \sigma(-y\cdot z)=-(1-\nu)\cdot\E\limits_{\sim x,y|\emptyset} y\on{x} \sigma(-y\cdot z) -\nu\cdot\E\limits_{\sim x,y|\neg\emptyset} y\on{x} \sigma(-y\cdot z)\label{eqn: grad on-manifold}
 \end{align}
 \begin{align}
     \E\limits_{\sim x,y|\emptyset} y\on{x} \sigma(-y\cdot z)=\pi\cdot\E\limits_{\sim x|y=1,\emptyset} \on{x} \sigma(-z)+(1-\pi)\cdot\E\limits_{\sim x|y=-1,\emptyset} -\on{x} \sigma(z)\label{eqn: grad separate}
 \end{align}
 \begin{align}
     \E\limits_{\sim x|y=1,\emptyset}\on{x}\sigma(-z)
    &=[\int_{0}^{l-k}\dots\int_{0}^{l-k}{\on{x}}_{i}\sigma(-z)\cdot {(l-k)}^{-d}]\, dx \hspace{0.3cm} \text{, $i\in\{1,\dots ,d\}$}\label{eqn: y=1,separate,integral}\\
    &=\inner{\Vec{c_1}}{\onefunc_{d}.\nicefrac{l-k}{2}}\label{eqn: y=1,separate}
 \end{align}
In \eqref{eqn: y=1,separate,integral} every integral is in a positive domain, with $\sigma(-z)\in (0,1)$ within the compact space. $0<\int_0^{l-k}x_{\shortparallel_i}\sigma(-\theta_{\shortparallel_i}x_{\shortparallel_i})< \int_0^{l-k}{\on{x}}_i=\nicefrac{l-k}{2}$, component-wise we get \eqref{eqn: y=1,separate} where $0<{\Vec{c_1}}_i<1$ for $i\in\{1,\dots, d\}$.
\begin{align}
     \E\limits_{\sim x|y=-1,\emptyset}-\on{x}\sigma(z)
    &=[\int_{-(l-k)}^{0}\dots\int_{-(l-k)}^{0}-{\on{x}}_{i}\sigma(z)\cdot {(l-k)}^{-d}]\, dx \hspace{0.3cm} \text{, $i\in\{1,\dots ,d\}$}\nonumber\\
    &=[\int_{0}^{-(l-k)}\dots\int_{0}^{-(l-k)}{\on{x}}_{i}\sigma(z)\cdot {(l-k)}^{-d}]\,dx
    \nonumber\\
    &=[\int_{0}^{l-k}\dots\int_{0}^{l-k}{\on{x}}_{i}\sigma(-z)\cdot {(l-k)}^{-d}]\,dx\overset{\eqref{eqn: y=1,separate,integral}}{=}
    \inner{\Vec{c_1}}{\onefunc_{d}.\nicefrac{l-k}{2}}\label{eqn: y=-1,separate}
 \end{align}
Substituting \eqref{eqn: y=1,separate} and \eqref{eqn: y=-1,separate} in \eqref{eqn: grad separate} we have:
\begin{equation}
    \E\limits_{\sim x,y|\emptyset} y\on{x} \sigma(-y\cdot z)=\inner{\Vec{c_1}}{\onefunc_{d}.\nicefrac{l-k}{2}}\label{eqn: grad separate 2}
\end{equation}
\begin{align}
     \E\limits_{\sim x|y=1,\neg\emptyset}\on{x}\sigma(-z)
    &=[\int_{-k}^{0}\dots\int_{-k}^{0}{\on{x}}_{i}\sigma(-z)\cdot {k}^{-d}]\, dx \hspace{0.3cm} \text{, $i\in\{1,\dots ,d\}$}\label{eqn: y=1,overlap,integral}\\
    \intertext{Using analogous arguments as in \eqref{eqn: y=1,separate,integral}}
    &=-\inner{\Vec{c_2}}{\onefunc_{d}.\nicefrac{k}{2}}\label{eqn: y=1,overlap}
 \end{align}
 Similarly, we have \begin{equation}
     \E\limits_{\sim x|y=-1,\neg\emptyset}-\on{x}\sigma(z)= -\inner{\Vec{c_2}}{\onefunc_{d}.\nicefrac{k}{2}}\label{eqn: y=-1,overlap}
 \end{equation} 
Substitute \eqref{eqn: y=1,overlap} and \eqref{eqn: y=-1,overlap} in \begin{align}
     \E\limits_{\sim x,y|\neg\emptyset} y\on{x} \sigma(-y\cdot z)=\pi\cdot\E\limits_{\sim x|y=1,\neg\emptyset} \on{x} \sigma(-z)+(1-\pi)\cdot\E\limits_{\sim x|y=-1,\neg\emptyset} -\on{x} \sigma(z)\label{eqn: grad overlap}
 \end{align}
 to get 
 \begin{equation}
    \E\limits_{\sim x,y|\neg\emptyset} y\on{x} \sigma(-y\cdot z)=-\inner{\Vec{c_2}}{\onefunc_{d}.\nicefrac{k}{2}}\label{eqn: grad overlap 2}
\end{equation}
 Using \eqref{eqn: grad separate 2},\eqref{eqn: grad overlap 2} in \eqref{eqn: grad on-manifold} we have:
  \begin{align}
     \nabla_{\on{\theta}}\cl(\gamma)=-(1-\nu)\cdot\inner{\Vec{c_1}}{\onefunc_{d}.\nicefrac{l-k}{2}}+\nu\cdot\inner{\Vec{c_2}}{\onefunc_{d}.\nicefrac{k}{2}}\label{eqn: grad on-manifold 2}
 \end{align}
 \subsubsection{Off Manifold}
 \begin{align}
     \nabla_{\off{\theta}}\cl(\gamma)=-\E\limits_{\sim x,y} y\off{x} \sigma(-y\cdot z)=-\pi\cdot\E\limits_{\sim x|y=1} \off{x} \sigma(-z) -(1-\pi)\cdot\E\limits_{\sim x|y=-1} -\off{x} \sigma(z)\label{eqn: grad off-manifold}
 \end{align}
 \begin{align}
     \E\limits_{\sim x|y=1} \off{x} \sigma(-z)&=
     [\int_{{\on{\mu}^{(1)}}_g-\sqrt{3}\off{\sigma}}^{{\on{\mu}^{(1)}}_g+\sqrt{3}\off{\sigma}}\dots\int_{{\on{\mu}^{(1)}}_1-\sqrt{3}\off{\sigma}}^{{\on{\mu}^{(1)}}_1+\sqrt{3}\off{\sigma}}{\off{x}}_{i}\sigma(-z)\cdot {(2\sqrt{3}\off{\sigma})}^{-g}]\, dx \hspace{0.3cm} \text{, $i\in\{1,\dots,g\}$}\label{eqn: y=1,off-manifold,integral}\\
     \intertext{Let $\off{x}=\off{\mu^{(1)}}+\eta$, then}&=
      [\int_{-\sqrt{3}\off{\sigma}}^{\sqrt{3}\off{\sigma}}\dots\int_{-\sqrt{3}\off{\sigma}}^{\sqrt{3}\off{\sigma}}{(\off{\mu^{(1)}}+\eta)}_{i}\sigma(-z)\cdot {(2\sqrt{3}\off{\sigma})}^{-g}]\, d\eta\nonumber\\
      &=\inner{\vec{c_3}}{\off{\mu}^{(1)}}+[\int_{0}^{\sqrt{3}\off{\sigma}}\dots\int_{0}^{\sqrt{3}\off{\sigma}}{\eta}_{i}\sigma(-z)\cdot {(2\sqrt{3}\off{\sigma})}^{-g}]\, d\eta\nonumber\\
      &+
      [\int_{-\sqrt{3}\off{\sigma}}^{0}\dots\int_{-\sqrt{3}\off{\sigma}}^{0}{\eta}_{i}\sigma(-z)\cdot {(2\sqrt{3}\off{\sigma})}^{-g}]\, d\eta\nonumber\\
      &=\inner{\vec{c_3}}{\off{\mu}^{(1)}}+\inner{(\vec{c_4}+\vec{c_5})}{\onefunc_g\cdot\nicefrac{\off{\sigma}}{4}}\label{eqn: y=1 off}
 \end{align}
 Also,
 \begin{align}
     \E\limits_{\sim x|y=-1} -\off{x} \sigma(z)&=\E\limits_{\sim x|y=-1} -\off{x} +\E\limits_{\sim x|y=-1} \off{x}\sigma(-z)\nonumber\\
     \intertext{Following similar steps following \eqref{eqn: y=1,off-manifold,integral} to \eqref{eqn: y=1 off}}
     &=-{\off{\mu}^{(-1)}}+\inner{\vec{c_6}}{\off{\mu}^{(-1)}}+\inner{(\vec{c_4}+\vec{c_5})}{\onefunc_g\cdot\nicefrac{\off{\sigma}}{4}}\label{eqn: y=-1 off}
 \end{align}
 Substituting \eqref{eqn: y=1 off}, \eqref{eqn: y=-1 off} in \eqref{eqn: grad off-manifold} we have:
 \begin{equation*}
     \nabla_{\off{\theta}}\cl(\gamma)=-\pi\left(\inner{\vec{c_3}}{\off{\mu}^{(1)}}\right)-(1-\pi)\left(\inner{-(\onefunc_g-\vec{c_6})}{\off{\mu}^{(-1)}}\right)-\inner{(\vec{c_4}+\vec{c_5})}{\onefunc_g\cdot\nicefrac{\off{\sigma}}{4}}
 \end{equation*}
\\
\subsection{\large Progressive Bound}
\subsubsection{Loss change all directions}
\label{subsubsec: Loss change all}$\Delta \cl (t)=\cl(\iter{\theta}{t+1})-\cl(\iter{\theta}{t})$. Using Taylor's expansion of $\cl(\iter{\theta}{t+1})$ around $\iter{\theta}{t}$ with Lipschitz smoothness \eqref{eqn: hessian bounds}, we have:
\begin{align*}
\cL({\iter{\theta}{t+1}} )&<\cL({\iter{\theta}{t}} )+\nabla_{\theta}\cL({\iter{\theta}{t}} )^T(\iter{\theta}{t+1}-\iter{\theta}{t})+\frac{c_u}{2}(\iter{\theta}{t+1}-\iter{\theta}{t})^{T}\begin{pmatrix}
    \on\sigma^2\bld{I}_d& \bld{0}\\
    \bld{0}& \off\sigma^2\bld{I}_g
\end{pmatrix}(\iter{\theta}{t+1}-\iter{\theta}{t}) \\
\intertext{As $\off\sigma<\on\sigma$, we have:}
   \cL({\iter{\theta}{t+1}} )&<\cL({\iter{\theta}{t}} )+\nabla_{\theta}\cL({\iter{\theta}{t}} )^T(\iter{\theta}{t+1}-\iter{\theta}{t})+\frac{c_u}{2}(\iter{\theta}{t+1}-\iter{\theta}{t})^{T}\on\sigma^2\bld{I}_D(\iter{\theta}{t+1}-\iter{\theta}{t}) 
   \end{align*}
   \begin{enumerate}
       \item For \textbf{logistic case} the gradient step is $\iter{\theta}{t+1}=\iter{\theta}{t}-\alpha\nabla_{\theta}\cl(\iter{\theta}{t})$
       \begin{align*}
    \cL({\iter{\theta}{t+1}} )&<\cL({\iter{\theta}{t}} )-\alpha\snorm{\nabla_{\theta}\cL({\iter{\theta}{t}})}^2+\alpha^2\frac{c_u\on{\sigma^2}}{2}\snorm{\nabla_{\theta}\cL({\iter{\theta}{t}})}^2\\
    \Delta \cl (t)&<-\alpha\snorm{\nabla_{\theta}\cL({\iter{\theta}{t}})}^2\left(1-\alpha\cdot\frac{c_u\on{\sigma^2}}{2}\right)\\
    \intertext{For step-size \begin{equation}
        \alpha\leq \frac{1}{c_u\on{\sigma^2}}
        \label{eqn: step-size logistic}
    \end{equation}}
    \Delta \cl (t)&<-(2c_u\on{\sigma^2})^{-1}\snorm{\nabla_{\theta}\cL({\iter{\theta}{t}})}^2
\end{align*}
\label{item: logistic pb all}
       \item For \textbf{$w-$step} \eqref{eqn: AGD 2 layer w} we have $w^{(t/2+1)}=w^{(t/2)}-\alpha_1\nabla_{w}\cl(\iter{\theta}{t})$. From lemma \ref{lem: reparametrize}, section:\ref{subsec: w-step}%\eqref{eqn: reparemtrize w}
       we have reparametrization $(\theta,\bld{0}_{m-D})=\iter{\wt}{t/2}=\bld{B}^T\iter{w}{t/2}$. Note that $\bld{B}=\begin{bmatrix}
           \iter{\bld{A}}{t/2}& \bld{C}
       \end{bmatrix}$ for some specific choice of $\bld{C}$. This induces a gradient descent step in $\wt$ or $\theta$ space as follows $\iter{\wt}{t/2+1}=\iter{\wt}{t/2}-\alpha_1\bld{B}^T\bld{B}\nabla_{\wt}\cl(\iter{\theta}{t})$ or $\iter{\theta}{t+1}=\iter{\theta}{t}-\alpha_1{\iter{\bld{A}}{t/2}}^T\iter{\bld{A}}{t/2}\nabla_{\theta}\cl(\iter{\theta}{t})$. %Here, $(\bld{B}^T\bld{B})_D$ is the first $D\times D$ square block of $\bld{B}^T\bld{B}$ corresponding to the identifiable parameter $\theta$. 
       Hence, the Taylor expansion can be written as:
       \begin{align*}
    \cL({\iter{\theta}{t+1}} )&<\cL({\iter{\theta}{t}} )-\alpha_1\nabla_{\theta}\cL({\iter{\theta}{t}})^T{\iter{\bld{A}}{t/2}}^T\iter{\bld{A}}{t/2}\nabla_{\theta}\cL({\iter{\theta}{t}})+\alpha_1^2\frac{c_u\on{\sigma^2}}{2}\snorm{{\iter{\bld{A}}{t/2}}^T\iter{\bld{A}}{t/2}\nabla_{\theta}\cL({\iter{\theta}{t}})}^2\\
    \intertext{let $\hat{\lambda}_{\bld{A}},\check{\lambda}_{\bld{A}}$ be the minimum/maximum eigen values of ${\iter{\bld{A}}{t/2}}^T\iter{\bld{A}}{t/2}$ respectively. Then :}
    \cL({\iter{\theta}{t+1}} )&<\cL({\iter{\theta}{t}} )-\alpha_1\hat{\lambda}_{\bld{A}}\snorm{\nabla_{\theta}\cL({\iter{\theta}{t}})}^2+\alpha_1^2\frac{c_u\on{\sigma^2}}{2}\check{\lambda}_{\bld{A}}^2\snorm{\nabla_{\theta}\cL({\iter{\theta}{t}})}^2\\
    %\intertext{$\snorm{(\bld{B}^T\bld{B})_D}<M$ for some universal $M$ as we are working with compact spaces}
    \Delta \cl (t)&<-\alpha_1\cdot\hat{\lambda}_{\bld{A}}\snorm{\nabla_{\theta}\cL({\iter{\theta}{t}})}^2\left(1-\alpha_1\cdot \frac{\check{\lambda}_{\bld{A}}^2c_u\on{\sigma^2}}{2\hat{\lambda}_{\bld{A}}}\right)\\
    \intertext{For step-size \begin{equation}
        \alpha_1\leq (M\cdot c_u\on{\sigma^2})^{-1}
        \label{eqn: step-size w step}
    \end{equation} Where $\nicefrac{\check{\lambda}_{\bld{A}}^2}{\hat{\lambda}_{\bld{A}}}\leq M$ for all $t$ and some $M<\infty$ (compact space).}
    \Delta \cl (t)&<-\hat{\lambda}_{\bld{A}}\cdot(2M\cdot c_u\on{\sigma^2})^{-1}\snorm{\nabla_{\theta}\cL({\iter{\theta}{t}})}^2
\end{align*}
\label{item: w-step pb all}
\item For \textbf{$\bld{A}-$step} \eqref{eqn: AGD 2 layer A} we have $\iter{\bld{A}}{\nicefrac{t+1}{2}}=\iter{\bld{A}}{\nicefrac{t-1}{2}}-\alpha_2\nabla_{\bld{A}}\cl(\iter{\theta}{t})$. From lemma \ref{lem: reparametrize}, section \ref{subsec: A-step}; %\eqref{eqn: reparemtrize A} we have reparametrization $
%(\theta,\bld{0}_{mD-D})=
$\vect{\bld{A}}^{(\nicefrac{t-1}{2})}=\vect\bld{U}^T{\At}^{(\nicefrac{t-1}{2})}=(\bld{I}_D\otimes\bld{U}^T)\vect{\At}^{(\nicefrac{t-1}{2})}$. Let $\bld{K}^T=(\bld{I}_D\otimes\bld{U}^T)^{-1}=\bld{I}_D\otimes Diag\left(\snorm{\iter{w}{\nicefrac{t+1}{2}}}^2,1,\dots,1\right)\bld{U}$ as $\bld{U}^TDiag\left(\snorm{\iter{w}{\nicefrac{t+1}{2}}}^{2},1,\dots,1\right)\bld{U}=\bld{I}_m$. This induces a gradient descent step in $\vect\At$ or $\theta$ space as follows $\iter{\vect\At}{\nicefrac{t+1}{2}}=\iter{\vect\At}{\nicefrac{t-1}{2}}-\alpha_2{\bld{K}^{T}}\bld{K}\nabla_{\At}\cl(\iter{\theta}{t})$ (where $\bld{K}^{T}\bld{K}=\bld{I}_D\otimes Diag\left(\snorm{\iter{w}{\nicefrac{t+1}{2}}}^4,1,\dots,1\right)\bld{I}_m$) or $\iter{\theta}{t+1}=\iter{\theta}{t}-\alpha_2\snorm{\iter{w}{\nicefrac{t+1}{2}}}^4\nabla_{\theta}\cl(\iter{\theta}{t})$. %Here, $({\bld{K}^{T}}\bld{K})_D$ is the first $D\times D$ square block of ${\bld{K}^{T}}\bld{K}$.
Following analogous steps to the derivation for the logistic case (Item \ref{item: logistic pb all}), with step size \begin{equation}
\alpha_2=(W\cdot c_u\on{\sigma^2})^{-1}
    \label{eqn: step-size A step}
\end{equation}, where $\snorm{\iter{w}{\nicefrac{t+1}{2}}}^4<W$ for all $t$. We get the following:  
       \begin{align*}
    \Delta \cl (t)&<-\snorm{\iter{w}{\nicefrac{t+1}{2}}}^4\cdot(2\cdot W c_u\on{\sigma^2})^{-1}\snorm{\nabla_{\theta}\cL({\iter{\theta}{t}})}^2
\end{align*}
       \label{item: A-step pb all}
   \end{enumerate}
   
    \subsubsection{Loss change on-manifold direction}
    \label{subsubsec: loss change on}
    $\on{\Delta} \cl (t)=\cl(\iter{\on{\theta}}{t+1},\iter{\off{\theta}}{t})-\cl(\iter{\theta}{t})$. Using Taylor's expansion of $\cl(\iter{\on{\theta}}{t+1},\iter{\off{\theta}}{t})$ around $\iter{\theta}{t}$ with Lipschitz smoothness \eqref{eqn: hessian bound on-manifold}, we have:
\begin{align*}
   \cl(\iter{\on{\theta}}{t+1},\iter{\off{\theta}}{t})&<\cl(\iter{\on{\theta}}{t},\iter{\off{\theta}}{t})+\nabla_{\on\theta}\cL({\iter{\theta}{t}} )^T(\iter{\on\theta}{t+1}-\iter{\on\theta}{t})+\frac{c_u}{2}(\iter{\on\theta}{t+1}-\iter{\on\theta}{t})^{T}\on\sigma^2\bld{I}_D(\iter{\on\theta}{t+1}-\iter{\on\theta}{t}) 
\end{align*}
\begin{enumerate}
    \item \textbf{Logistic case:} From section \ref{subsubsec: Loss change all} Item \ref{item: logistic pb all} remember that the gradient step for logistic case is $\iter{\theta}{t+1}=\iter{\theta}{t}-\alpha\nabla_{\theta}\cl(\iter{\theta}{t})$ with step size $\alpha\leq \frac{1}{c_u\on{\sigma^2}}$\eqref{eqn: step-size logistic}. Hence, $\iter{\on{\theta}}{t+1}=\iter{\on{\theta}}{t}-\alpha\nabla_{\on{\theta}}\cl(\iter{\theta}{t})$.
    \begin{align*}
        \on\Delta\cl(t)&<-\alpha\snorm{\nabla_{\on\theta}\cL({\iter{\theta}{t}})}^2 +\alpha^2\frac{c_u\cdot \on\sigma^2}{2}\snorm{\nabla_{\on\theta}\cL({\iter{\theta}{t}})}^2\\
        \intertext{As $\alpha\leq \frac{1}{c_u\on{\sigma^2}}$}
        \on\Delta\cl(t)&<-(2c_u\on{\sigma^2})^{-1}\snorm{\nabla_{\on{\theta}}\cL({\iter{\theta}{t}})}^2
    \end{align*}
    \label{item: logistic pb on}
    \item \textbf{$w$-step:} From section \ref{subsubsec: Loss change all} Item \ref{item: w-step pb all} remember that the gradient step induced in the identifiable parameter $\theta$, for $w$-step is $\iter{\theta}{t+1}=\iter{\theta}{t}-\alpha_1{\iter{\bld{A}}{t/2}}^T\iter{\bld{A}}{t/2}\nabla_{\theta}\cl(\iter{\theta}{t})$ with step size $\alpha_1\leq (M\cdot c_u\on{\sigma^2})^{-1}$\eqref{eqn: step-size w step}. Due to orthogonalization of the first layer, ${\iter{\bld{A}}{t/2}}^T\iter{\bld{A}}{t/2}=\bld{I}_D$ and $\hat{\lambda}_{\bld{A}},\check{\lambda}_{\bld{A}}=1$ with $M=1$.  Hence, $\iter{\on{\theta}}{t+1}=\iter{\on{\theta}}{t}-\alpha_1\nabla_{\on{\theta}}\cl(\iter{\theta}{t})$ and $\alpha_1\leq (c_u\on{\sigma^2})^{-1}$. Now, mimicking the steps in the logistic case, we have:
    $$\on\Delta\cl(t)<-(2c_u\on{\sigma^2})^{-1}\snorm{\nabla_{\on{\theta}}\cL({\iter{\theta}{t}})}^2$$
    \label{item: w-step pb on}
    \item \textbf{$\bld{A}$-Step}
    From section \ref{subsubsec: Loss change all} Item \ref{item: A-step pb all} remember that the gradient step induced in the identifiable parameter $\theta$, for $\bld{A}$-step is $\iter{\theta}{t+1}=\iter{\theta}{t}-\alpha_2\snorm{\iter{w}{\nicefrac{t+1}{2}}}^4\nabla_{\theta}\cl(\iter{\theta}{t})$ with step size $\alpha_2=(W\cdot c_u\on{\sigma^2})^{-1}$\eqref{eqn: step-size A step}. Hence, $\iter{\on{\theta}}{t+1}=\iter{\on{\theta}}{t}-\alpha_2\snorm{\iter{w}{\nicefrac{t+1}{2}}}^4\nabla_{\theta}\nabla_{\on{\theta}}\cl(\iter{\theta}{t})$. Now, following identical the steps as in the logistic case, $w$-step, we have:
    $$\on\Delta\cl(t)<-\snorm{\iter{w}{\nicefrac{t+1}{2}}}^4\cdot(2\cdot W c_u\on{\sigma^2})^{-1}\snorm{\nabla_{\theta}\cL({\iter{\theta}{t}})}^2$$
    \label{item: A-step pb on}
\end{enumerate}
\subsubsection{Loss change off-manifold directions}
\label{subsubsec: loss change off}
$\off{\Delta} \cl (t)=\cl(\iter{\on{\theta}}{t},\iter{\off{\theta}}{t+1})-\cl(\iter{\theta}{t})$. Using Taylor's expansion of $\cl(\iter{\on{\theta}}{t},\iter{\off{\theta}}{t+1})$ around $\iter{\theta}{t}$ with Lipschitz smoothness \eqref{eqn: hessian bound off-manifold}, we have:
\begin{align*}
 \cl(\iter{\on{\theta}}{t},\iter{\off{\theta}}{t+1})&<\cl(\iter{\on{\theta}}{t},\iter{\off{\theta}}{t})+\nabla_{\off\theta}\cL({\iter{\theta}{t}} )^T(\iter{\off\theta}{t+1}-\iter{\off\theta}{t})+\frac{c_u}{2}(\iter{\off\theta}{t+1}-\iter{\off\theta}{t})^{T}\off\sigma^2\bld{I}_D(\iter{\off\theta}{t+1}-\iter{\off\theta}{t})  \\
\intertext{As $\off\sigma<\on\sigma$, we have:}
    \cl(\iter{\on{\theta}}{t},\iter{\off{\theta}}{t+1})&<\cl(\iter{\on{\theta}}{t},\iter{\off{\theta}}{t})+\nabla_{\off\theta}\cL({\iter{\theta}{t}} )^T(\iter{\off\theta}{t+1}-\iter{\off\theta}{t})+\frac{c_u}{2}(\iter{\off\theta}{t+1}-\iter{\off\theta}{t})^{T}\on\sigma^2\bld{I}_D(\iter{\off\theta}{t+1}-\iter{\off\theta}{t}) 
   \end{align*}
   Follow identical steps corresponding to off-manifold direction as in section \ref{subsubsec: loss change on} to get:
   \begin{align*}
   \text{(logistic step) \hspace{0.1cm}}\on\Delta\cl(t)&<-(2c_u\on{\sigma^2})^{-1}\snorm{\nabla_{\on{\theta}}\cL({\iter{\theta}{t}})}^2\\
   \text{(w-step) \hspace{0.1cm}}\on\Delta\cl(t)&<-(2c_u\on{\sigma^2})^{-1}\snorm{\nabla_{\on{\theta}}\cL({\iter{\theta}{t}})}^2
   \\\text{(A-step) \hspace{0.1cm}}
       \on\Delta\cl(t)&<-\snorm{\iter{w}{\nicefrac{t+1}{2}}}^4\cdot(2\cdot W c_u\on{\sigma^2})^{-1}\snorm{\nabla_{\theta}\cL({\iter{\theta}{t}})}^2
   \end{align*}
\end{proof}

\begin{lemma}
\label{lem: coordinate wise optimality}
    Let $\theta^*=(\on\theta^*,\off\theta^*)$ be the optimal identifiable parameter minimizing $\cl(\gamma)$ and $\iter{\theta}{t}$ be the $t^{th}$ iterate of $\theta$ induced by the optimization in logistic regression (Section \ref{subsec: logistic setup}) and 2-Linear Layer network setup (Section \ref{subsec: 2layer setup}), then for $t_2>t_1$ and appropriate $\alpha,\alpha_1,\alpha_2 \preceq \on\sigma^{-2}$ we have:
    \begin{align}
        \cl(\on\theta,\iter{\off{\theta}}{t_2})&\leq \cl(\on\theta,\iter{\off{\theta}}{t_1});\,\forall\on{\theta}\label{eqn: component wise optimality off}\\
        \cl(\on\theta^{(t_2)},\off\theta)&\leq \cl(\on\theta^{(t_1)},\off{\theta}) ;\forall \off{\theta} \label{eqn: component wise optimality on}
    \end{align}
    In particular, we have  component-wise progressive bounds as such:
    \begin{align}
\cl(\on{\theta},\iter{\off{\theta}}{t+1})-\cl(\on{\theta},\iter{\off{\theta}}{t})&\leq c_p(2c_u\off\sigma^2)^{-1}\snorm{\nabla_{\off\theta}\cl(\iter{\theta}{t})}^2\label{eqn: pb off component}\\
\cl(\iter{\on{\theta}}{t+1},\off{\theta})-\cl(\iter{\on{\theta}}{t},\off{\theta})&\leq c_p(2c_u\on\sigma^2)^{-1}\snorm{\nabla_{\on\theta}\cl(\iter{\theta}{t})}^2\label{eqn: pb on component}
    \end{align}
\end{lemma}
\begin{proof}
We will prove for the $\perp$ direction; one can identically derive the result for the $\shortparallel$ direction.  Suppose we do $T$ iteration of optimization; this induces a sequence in off-manifold parameters $\off\theta$. $\iter{\off\theta}{0}\to \iter{\off\theta}{1}\to\dots\to \iter{\off\theta}{T}$. Fixing the $\shortparallel$ parameter to general $\on\theta$, apply Taylor expansion with Lipschitz smoothness (lemma \ref{lem:lipschitz/strong convex}) to the loss $\cl(\on\theta,\iter{\off\theta}{t+1})$ around the point $(\on\theta,\iter{\off\theta}{t})$:
\begin{align*}
    \cl(\on{\theta},\iter{\off{\theta}}{t+1})
    &\leq\cl(\on{\theta},\iter{\off{\theta}}{t})+\nabla_{\off\theta}\cl(\on\theta,\iter{\off\theta}{t})^T(\iter{\off\theta}{t+1}-\iter{\off\theta}{t})+\frac{c_u\off\sigma^2}{2}\snorm{\iter{\off\theta}{t+1}-\iter{\off\theta}{t}}^2\\
    \intertext{$\iter{\off\theta}{t+1}-\iter{\off\theta}{t}=-\breve{\alpha}\nabla_{\off\theta}\cl(\iter{\theta}{t})$ (\eqref{eqn: general alpha}) for some step-size $\Breve{\alpha}$ depending on the situation.}
\cl(\on{\theta},\iter{\off{\theta}}{t+1})
    &\leq\cl(\on{\theta},\iter{\off{\theta}}{t})-\breve\alpha\nabla_{\off\theta}\cl(\on\theta,\iter{\off\theta}{t})^T\nabla_{\off\theta}\cl(\iter{\theta}{t})+\breve\alpha^2\frac{c_u\off\sigma^2}{2}\snorm{\nabla_{\off\theta}\cl(\iter{\theta}{t})}^2\\
    &\overset{(i)}{\leq}\cl(\on{\theta},\iter{\off{\theta}}{t})-\breve\alpha\cdot c\snorm{\nabla_{\off\theta}\cl(\iter{\theta}{t})}^2+\breve\alpha^2\frac{c_u\off\sigma^2}{2}\snorm{\nabla_{\off\theta}\cl(\iter{\theta}{t})}^2
    \\
    &\leq\cl(\on{\theta},\iter{\off{\theta}}{t})-\breve{\alpha}\snorm{\nabla_{\off\theta}\cl(\iter{\theta}{t})}^2\underbrace{\left(c-\breve\alpha\frac{c_u\off\sigma^2}{2}\right)}_{\overset{(ii)}{\geq 0}}\\
    &\leq \cl(\on{\theta},\iter{\off{\theta}}{t})-c_p(2c_u\on\sigma^2)^{-1}\snorm{\nabla_{\off\theta}\cl(\iter{\theta}{t})}^2
    \intertext{$c_p=c$ when $c\geq 1$ and $c_p=2c-1$ when $c<1$.}
\end{align*}
$(i):$ $\nabla_{\off\theta}\cl(\on\theta,\iter{\off\theta}{t})=-\E\limits_{\sim x,y} y\off{x} \sigma(-y\cdot z)=-\E\limits_{\sim x,y} y\off{x}(1+\exp{(y\on\theta^T\on x)}\cdot\exp{(y{\iter{\off\theta}{t}}^T\off x)} )^{-1}$ 
Let $r=\exp{(y\on\theta^T\on x)}\cdot\exp{(-y{\iter{\on\theta}{t}}^T\on x)}$. Then $\nabla_{\off\theta}\cl(\on\theta,\iter{\off\theta}{t})=-\E\limits_{\sim x,y} y\off{x}(1+r\exp{(y{\iter{\on\theta}{t}}^T\on x)}\cdot\exp{(y{\iter{\off\theta}{t}}^T\off x)} )^{-1}\underset{\text{As }r>0}{=}-c\cdot\E\limits_{\sim x,y} y\off{x}(1+\exp{(y{\iter{\on\theta}{t}}^T\on x)}\cdot\exp{(y{\iter{\off\theta}{t}}^T\off x)} )^{-1}=c\cdot \nabla_{\off\theta}\cl(\iter{\theta}{t})$ for some $c>0$.\\\\
$(ii):$ If $c\geq 1$  with $\breve\alpha\leq (c_u\on\sigma^2)^{-1}$ \eqref{eqn: bound on general alpha}, then $(ii)>0$. If $0<c<1$, one can choose a strictly smaller step size still satisfying the upper bound of \eqref{eqn: bound on general alpha} $\breve\alpha\leq c\cdot(c_u\on\sigma^2)^{-1}< (c_u\on\sigma^2)^{-1}$ which makes $(ii)>0$. \\
As $\cl(\on{\theta},\iter{\off{\theta}}{t+1})
    \leq\cl(\on{\theta},\iter{\off{\theta}}{t})$, for all $t$ using induction $\cl(\on{\theta},\iter{\off{\theta}}{t_2})
    \leq\cl(\on{\theta},\iter{\off{\theta}}{t_1})$ when $t_2>t_1$. One can identically show that $\cl(\on\theta^{(t_2)},\off\theta)\leq \cl(\on\theta^{(t_1)},\off{\theta})$.
\end{proof}
\subsection{Convergence Theorems}
\begin{proof}[Proof of Theorem \ref{thm: parameter convergence}]
We start with the proof of $\perp$ direction; one can derive the result for $\shortparallel$ mimicking similar steps with minor alteration. $\theta^*=(\on\theta^*,\off\theta^*)$ is the optimal identifiable parameter value which minimizes the loss and takes $T$ iteration of optimization.\\
\textbf{Upper bound the Gradient norm by Loss difference:} Using lemma \ref{lem: coordinate wise optimality} with $t_2=T,t_1=t+1$:
\begin{align*}
    \cl(\iter{\on\theta}{t},\off\theta^*)&\leq \cl(\iter{\on{\theta}}{t},\iter{\off{\theta}}{t+1})\\
    \intertext{ Using lipschitz smoothness in $\perp$ direction we have}
    &\leq\cl(\iter{\on{\theta}}{t},\iter{\off{\theta}}{t})+\nabla_{\off\theta}\cl(\iter{\theta}{t})^T(\iter{\off\theta}{t+1}-\iter{\off\theta}{t})+\frac{c_u\off\sigma^2}{2}\snorm{\iter{\off\theta}{t+1}-\iter{\off\theta}{t}}^2
    \intertext{Gradient descent steps induce updates in $\perp$ of the form, $\iter{\off\theta}{t+1}-\iter{\off\theta}{t}=-\breve{\alpha}\nabla_{\off\theta}\cl(\iter{\theta}{t})$}
    &\leq\cl(\iter{\on{\theta}}{t},\iter{\off{\theta}}{t})-\breve{\alpha}\snorm{\nabla_{\off\theta}\cl(\iter{\theta}{t})}^2\left(1-\breve\alpha\frac{c_u\off\sigma^2}{2}\right)
\end{align*}
\begin{equation}
     \implies \snorm{\nabla_{\off\theta}\cl(\iter{\theta}{t})}^2\leq\left(\cl(\iter{\theta}{t})-\cl(\iter{\on\theta}{t},\off\theta^*)\right)\left(1-\breve\alpha\frac{c_u\off\sigma^2}{2}\right)^{-1}\cdot{\breve\alpha}^{-1}\label{eqn: gradient off upperbound}
\end{equation}
Similarly, for the on-manifold direction
\begin{equation}
     \implies \snorm{\nabla_{\on\theta}\cl(\iter{\theta}{t})}^2\leq\left(\cl(\iter{\theta}{t})-\cl(\on\theta^*,\iter{\off\theta}{t})\right)\left(1-\on\alpha\frac{c_u\on\sigma^2}{2}\right)^{-1}\cdot{\breve\alpha}^{-1}\label{eqn: gradient on upperbound}
\end{equation}
Also, using strong convexity \eqref{eqn: hessian bound off-manifold} we have:
\begin{align}
    \cl(\iter{\on\theta}{t},\off\theta^*)&\geq\cl(\iter{\theta}{t})+\nabla_{\off\theta}\cl(\iter{\theta}{t})^T(\off\theta^*-\iter{\off\theta}{t})+\frac{c_l\off\sigma^2}{2}\snorm{\off\theta^*-\iter{\off\theta}{t}}^2\nonumber\\
    \nabla_{\off\theta}\cl(\iter{\theta}{t})^T(\iter{\off\theta}{t}-\off\theta^*)&\geq \cl(\iter{\theta}{t})- \cl(\iter{\on\theta}{t},\off\theta^*)+\frac{c_l\off\sigma^2}{2}\snorm{\off\theta^*-\iter{\off\theta}{t}}^2\label{eqn: off grad_optimaldiff bound}
\end{align}
Similarly, for the on-manifold direction, strong convexity leads to:
\begin{equation}
    \nabla_{\on\theta}\cl(\iter{\theta}{t})^T(\iter{\on\theta}{t}-\on\theta^*)\geq \cl(\iter{\theta}{t})- \cl(\on\theta^*,\iter{\off\theta}{t})+\frac{c_l\on\sigma^2}{2}\snorm{\on\theta^*-\iter{\on\theta}{t}}^2\label{eqn: on grad_optimaldiff bound}
\end{equation}
Also, from \eqref{eqn: step-size logistic}, \eqref{eqn: step-size w step}, \eqref{eqn: step-size A step}
\begin{equation}
   \breve\alpha = \left\{
\begin{array}{ll}
      \alpha & ,\alpha\leq {(c_u\on\sigma^2)}^{-1} \,\text{(logistic)}\\
      \alpha_1 & ,\alpha_1\leq {(c_u\on\sigma^2)}^{-1} \,\text{(w-step)}\\
      \alpha_2 \snorm{\iter{w}{\nicefrac{t+1}{2}}}^4 &, \alpha_2 \leq {(W\cdot c_u\on{\sigma}^2)}^{-1},\snorm{\iter{w}{\nicefrac{t+1}{2}}}^4<W, \,\forall t \,\text{(A-step)}
\end{array} 
\right. 
\label{eqn: general alpha}
\end{equation}
\begin{equation}
    \breve\alpha\leq (c_u\on\sigma^2)^{-1} \label{eqn: bound on general alpha}
\end{equation}
The parameter difference can be expressed in terms of the difference at the previous iteration:
\begin{align*}
    \norm{\iter{\theta}{t+1}_\perp-\theta^*_{\perp}}^2&=\snorm{\iter{\theta}{t}_\perp-\breve\alpha\nabla_{\off\theta}\cL(\iter{\theta}{t})-\theta^*_{\perp}}^2\\
    &=\snorm{\iter{\theta}{t}_\perp-\theta^*_\perp}^2-2\breve\alpha\nabla_{\off\theta}\cL(\iter{\theta}{t})(\iter{\theta}{t}_\perp-\theta^*_\perp)+\breve\alpha^2\snorm{\nabla_{\off\theta}\cL(\iter{\theta}{t})}^2\\
    \intertext{Substituting \eqref{eqn: off grad_optimaldiff bound}}
    &\leq \snorm{\iter{\theta}{t}_\perp-\theta^*_\perp}^2-2\breve\alpha\left(\cl(\iter{\theta}{t})- \cl(\iter{\on\theta}{t},\off\theta^*)+\frac{c_l\off\sigma^2}{2}\snorm{\iter{\theta}{t}_\perp-\theta^*_\perp}^2\right)+\breve\alpha^2\snorm{\nabla_{\off\theta}\cL(\iter{\theta}{t})}^2\\
    \intertext{Using  \eqref{eqn: gradient off upperbound}}
    &\leq \snorm{\iter{\theta}{t}_\perp-\theta^*_\perp}^2-2\breve\alpha\left(\cl(\iter{\theta}{t})- \cl(\iter{\on\theta}{t},\off\theta^*)+\frac{c_l\off\sigma^2}{2}\snorm{\iter{\theta}{t}_\perp-\theta^*_\perp}^2\right)\\
    &+\left(\cl(\iter{\theta}{t})-\cl(\iter{\on\theta}{t},\off\theta^*)\right)\left(1-\breve\alpha\frac{c_u\off\sigma^2}{2}\right)^{-1}\cdot{\breve\alpha}\\
    \intertext{Using \eqref{eqn: bound on general alpha} and $\nicefrac{\off\sigma}{\on\sigma}<1$}
    &=\snorm{\iter{\theta}{t}_\perp-\theta^*_\perp}^2\left(1-\breve\alpha c_l\off\sigma^2\right)+\underbrace{\left(\cL(\iter{\theta}{t}_{\perp},\theta_{\shortparallel})-\cL(\theta^*_{\perp},\theta_{\shortparallel})\right)}_{\leq0}\cdot\left(\underbrace{(1-\breve\alpha\frac{c_u\off\sigma^2}{2})^{-1}}_{\leq2 }-2\right)\cdot\breve\alpha\\
    \intertext{Using \eqref{eqn: bound on general alpha} and $c_l<c_u$ \eqref{eqn: hessian bounds}.}
    &\leq \snorm{\iter{\theta}{t}_\perp-\theta^*_\perp}^2\underbrace{\left(1-\frac{c_l\off\sigma^2}{c_u\on\sigma^2}\right)}_{\leq 1}
\end{align*}
Suppose we have $T$ total iterations then, inductively the equations accumulate over as:
\begin{align}
    \snorm{\iter{\theta}{T}_\perp-\theta^*_{\perp}}^2&\leq\snorm{\iter{\theta}{0}_\perp-\theta^*_\perp}^2\left(1-\frac{c_l\off\sigma^2}{c_u\on\sigma^2}\right)^T\leq\snorm{\iter{\theta}{0}_\perp-\theta^*_\perp}^2\exp\left(-T\frac{c_l\off\sigma^2}{c_u\on\sigma^2}\right)
\end{align}
Hence, we require $T\geq \frac{c_u\on\sigma^2}{c_l\off\sigma^2}\log\left(\frac{\snorm{\iter{\theta}{0}_\perp-\theta^*_\perp}}{\delta}\right)$ for $\mathcal{O}(\delta)$ distance from $\theta^*_{\perp}$.\\

Similarly, one can derive the bound for the $\shortparallel$ direction mimicking the exact steps as $\perp$ direction but instead using analogous equations \eqref{eqn: on grad_optimaldiff bound},\eqref{eqn: gradient on upperbound}.\\
$T\geq \frac{c_u\on\sigma^2}{c_u\on\sigma^2}\log\left(\frac{\snorm{\iter{\theta}{0}_\shortparallel-\theta^*_\shortparallel}}{\delta}\right)=\log\left(\frac{\snorm{\iter{\theta}{0}_\shortparallel-\theta^*_\shortparallel}}{\delta}\right)$ for $\mathcal{O}(\delta)$ distance from $\theta^*_{\shortparallel}$.
\begin{remark}
\label{remark: variable step-size}
    When deriving the bounds for $\snorm{\iter{\theta}{0}_\shortparallel-\theta^*_\shortparallel}$ vs $\snorm{\iter{\theta}{0}_\perp-\theta^*_\perp}$, the key point of difference is using the same step size $\breve\alpha\preceq \on\sigma^{-2}$ even though the strong convexity constants are different, $\propto \on\sigma^2$ for $\shortparallel$ direction and $\propto\off\sigma^2$ for $\perp$ direction. This introduces the ratio $\nicefrac{\off\sigma}{\on\sigma}$ in the $\perp$ case, while for $\shortparallel$ case the step size and strong-convexity parameter neutralize each other to 1. Hence, if we could enforce separate step size for $\perp,\shortparallel$ directions $\alpha_\perp\preceq \off\sigma^{-2},\alpha_\shortparallel\preceq \on\sigma^{-2}$. Then, we can get equivalent rates in both directions.
\end{remark}

\end{proof}
\begin{proof}[Proof of Theorem \ref{thm: loss convergence}]
\label{proof: thm loss convergence}
We want to minimize the loss $\cl(\theta)=\E\limits_{\sim x,y}\ell(y\cdot z)$. The minimum loss that can be attained has a natural lower bound $\min\limits_{\theta}\cl(\theta)\geq (1-\nu)\cdot\min\limits_{\theta}\E\limits_{\sim x,y|\emptyset}\ell(y\cdot z)+\nu\min\limits_{\theta}\E\limits_{\sim x,y|\neg\emptyset}\ell(y\cdot z)$.
Suppose we only optimize the loss w.r.t.  $\on\theta$ and $\off\theta=\bld0$, then a perfect classifier on $\shortparallel$ direction can distinguish $\on x$ with probability 1 or 0 loss, but only with $\sfrac{1}{2}$ probability on $\perp$ direction. In this case,
$\min\limits_{\on\theta}\cl(\theta)\geq \nu\log 2$. In general, the classifier isn't perfect, and $\off\theta$ can be fixed at some default value. Hence, the lower bound is controlled by $\nu\log2$ up to a constant $\min\limits_{\on\theta}\cl(\theta)\geq C=\Omega(\nu\log 2)$.
Consider two cases when the loss tolerance $\delta< C$ or $>C$.
\subsection*{Case 1: $\delta<C$}
From the progressive bounds in proof of Theorem \ref{thm: Progressive Bound} for any optimization iterate (logistic, w-step or A-step) induced in the identifiable parameter $\theta$, has a decremental loss:
\begin{equation}
    \cl(\iter{\theta}{t+1})-\cl(\iter{\theta}{t})\leq -(2c_u\on\sigma^2)^{-1}\snorm{\nabla_{\theta}\cl(\iter{\theta}{t})}^2 \label{eqn: progressive bound all case}
\end{equation}
%\textbf{\underline{PL inequality:}}
\begin{equation}
    -2c_l\off{\sigma^2}\cdot\left(\cl(\iter{\theta}{t})-\cl(\theta^*)\right)\geq -\snorm{\nabla_{\theta}\cl(\iter{\theta}{t})}^2\tag{PL-inequality}\label{eqn: pl-inequality 1}
\end{equation}
\textit{Proof of PL inequality:} The \eqref{eqn: pl-inequality 1} is a consequence of strong convexity. Using strong convexity in the space of identifiable parameter $\theta$ Lemma \ref{lem:lipschitz/strong convex},\eqref{eqn: hessian bounds} we have:
\begin{align*}
 \cl(\theta)&\geq\cl(\iter{\theta}{t})+\nabla_{\theta}\cl(\iter{\theta}{t})^T(\theta-\iter{\theta}{t})+\frac{c_l}{2}(\theta-\iter{\theta}{t})^{T}\begin{pmatrix}
    \on\sigma^2\bld{I}_d& \bld{0}\\
    \bld{0}& \off\sigma^2\bld{I}_g
\end{pmatrix}(\theta-\iter{\theta}{t})\\
\cl(\theta)&\geq\cl(\iter{\theta}{t})+\nabla_{\theta}\cl(\iter{\theta}{t})^T(\theta-\iter{\theta}{t})+\frac{c_l\cdot\off{\sigma^2}}{2}\snorm{\theta-\iter{\theta}{t}}^2\\
\intertext{Minimizing both sides w.r.t $\theta$, happens for $\theta=\iter{\theta}{t}-\nabla_{\theta}\cl(\iter{\theta}{t}){(c_l\off\sigma^2)}^{-1}$}
\cl(\theta^*)&\geq\cl(\iter{\theta}{t})-\snorm{\nabla_{\theta}\cl(\iter{\theta}{t})}^2{(2c_l\off{\sigma}^2)}^{-1} \qed
\end{align*}
% \begin{enumerate}
%     \item \underline{Logistic step:} Using strong convexity 
%     \item \underline{w-step:}
%     \item \underline{A-step:}
% \end{enumerate}
Hence, from the progressive bound, we have \eqref{eqn: progressive bound all case} :
\begin{align*}
    \cl(\iter{\theta}{t+1})&\leq \cl(\iter{\theta}{t})-(2c_u\on\sigma^2)^{-1}\snorm{\nabla_{\theta}\cl(\iter{\theta}{t})}^2\\
    \intertext{Subtracting $\cl(\theta^*)$ from both sides}
    \cl(\iter{\theta}{t+1})-\cl(\theta^*)&\leq \cl(\iter{\theta}{t})-\cl(\theta^*)-(2c_u\on\sigma^2)^{-1}\snorm{\nabla_{\theta}\cl(\iter{\theta}{t})}^2\\
    \intertext{Using \eqref{eqn: pl-inequality 1}}
    \cl(\iter{\theta}{t+1})-\cl(\theta^*)&\leq \left(\cl(\iter{\theta}{t})-\cl(\theta^*)\right)\cdot\left(1-\frac{c_l\off\sigma^2}{c_u\on\sigma^2}\right)
\end{align*}
Suppose we have $T$ total iterations; then, inductively, the equations accumulate over as:
\begin{align}
    \cl(\iter{\theta}{T})-\cl(\theta^*)&\leq\left(\cl(\iter{\theta}{0})-\cl(\theta^*)\right)\left(1-\frac{c_l\off\sigma^2}{c_u\on\sigma^2}\right)^T\leq\left(\cl(\iter{\theta}{0})-\cl(\theta^*)\right)\exp\left(-T\frac{c_l\off\sigma^2}{c_u\on\sigma^2}\right)
\end{align}
Hence, we require $T\geq \frac{c_u\on\sigma^2}{c_l\off\sigma^2}\log\left(\left(\cl(\iter{\theta}{0})-\cl(\theta^*)\right)\delta^{-1}\right)$ for $\mathcal{O}(\delta)$ error tolerance. This is rate $r_2$ in the thm statement.

\subsection*{Case 2: $\delta>C$}
The rate $r_2$ proved in the previous case is universal and holds for this case as well. However, we can obtain a better rate in this scenario.\\
In this case, the loss can attain value $\delta$ solely by optimizing $\on\theta$. Therefore, we will upper bound the original gradient descent loss sequence by a loss sequence solely dependant on updates of $\on\theta$, and we will see that because convergence is better on $\on\theta$ direction, we can get better rates.\\
Note that if $\cL(\theta)=C$, then using Jensen's inequality for convex functions: $\ell(\E\limits_{\sim x,y}y\cdot z)\leq \cL(\theta)=C$.
\begin{align*}
    \ln{(1+\exp(-\bE y\theta^Tx))}&\leq C\\
    \exp(-\bE y\theta^Tx)&\leq e^C-1\\
    \bE y\theta^Tx&\geq \ln(\frac{1}{e^{C}-1})\\
    \bE y(\Inner{\theta_{\shortparallel}}{x_{\shortparallel}}+\Inner{\theta_{\perp}}{x_{\perp}})&\geq \ln(\frac{1}{e^{C}-1})\\
    \pi(\Inner{\theta_{\shortparallel}}{\iter{\mu}{1}_{\shortparallel}}+\Inner{\theta_{\perp}}{\iter{\mu}{1}_{\perp}})-(1-\pi)(\Inner{\theta_{\shortparallel}}{\iter{\mu}{-1}_{\shortparallel}}+\Inner{\theta_{\perp}}{\iter{\mu}{-1}_{\perp}})&\geq \ln(\frac{1}{e^{C}-1})
\end{align*}
Note that the above can always be satisfied by $\on\theta=c'\cdot\on\mu^{(1)}-c''\cdot\on\mu^{(-1)}$ for appropriate choice of constants $c',c''$. Hence for a fixed $\off\theta$ there always exists some $\on{\tilde{\theta}}$ such that $\cl(\on{\tilde{\theta}},\off\theta)<C\implies \cl(\on{\tilde{\theta}},\off\theta)-\cl(\theta^*)<C$. This means fixing $\off\theta=\iter{\off\theta}{0}$ at initialization, there exists $\on{\tilde{\theta}}$ as well which satisfies: %So, theoretically, if we restrict the gradient descent sequence to $\on{\theta}^{(t)}$, we can achieve 
\begin{equation}
\min\limits_{\on\theta}\cL(\on{\theta},\iter{\theta_{\perp}}{0})-\cl(\theta^*)\leq\cL(\on{\tilde{\theta}},\iter{\theta_{\perp}}{0})-\cl(\theta^*)<C\label{eqn: attainability via on theta} 
\end{equation} %for some $T$ (Every gradient descent step monotonically decreases the loss in each component lemma \ref{lem: coordinate wise optimality}).
Using lemma \ref{lem: coordinate wise optimality} we have:
\begin{align}
    \cL(\iter{\theta}{T})&\leq\cL(\iter{\theta_{\shortparallel}}{T},\iter{\theta_{\perp}}{0})\nonumber\\
    \intertext{Subtracting $\cl(\theta^*)$ both sides}
     \cL(\iter{\theta}{T})-\cl(\theta^*)&\leq\cL(\iter{\theta_{\shortparallel}}{T},\iter{\theta_{\perp}}{0})-\cl(\theta^*)\nonumber\\
     &= \cL(\iter{\theta_{\shortparallel}}{T},\iter{\theta_{\perp}}{0})-\min\limits_{\on\theta}\cL(\on{\theta},\iter{\theta_{\perp}}{0})+\min\limits_{\on\theta}\cL(\on{\theta},\iter{\theta_{\perp}}{0})-\cl(\theta^*)\nonumber\\
     &\leq (\delta-C)+C=\delta\label{eqn: loss bound >C}
\end{align}
Hence, if we find a $T$ for which $\cL(\iter{\theta_{\shortparallel}}{T},\iter{\theta_{\perp}}{0})-\min\limits_{\on\theta}\cL(\on{\theta},\iter{\theta_{\perp}}{0})\leq \delta-C$ then we are done.\\
From lemma \ref{lem: coordinate wise optimality},\eqref{eqn: pb on component} we have a progressive bound on the off-manifold component as follows:
\begin{equation}
    \cl(\iter{\on\theta}{t+1},\iter{\off\theta}{0})\leq \cl(\iter{\on\theta}{t},\iter{\off\theta}{0}) - c_p(2c_u\on\sigma^2)^{-1}\snorm{\nabla_{\on\theta}\cl(\iter{\theta}{t})}^2
    \label{eqn: pb on component 2}
\end{equation}
\begin{equation}
    -2c'_pc_l\on{\sigma^2}\cdot\left(\cl(\iter{\on\theta}{t},\iter{\off\theta}{0})-\min\limits_{\on\theta}\cL(\on{\theta},\iter{\theta_{\perp}}{0})\right)\geq -\snorm{\nabla_{\on\theta}\cl(\iter{\theta}{t})}^2\tag{PL-inequality $\on\theta$}\label{eqn: pl-inequality 2}
\end{equation}
\textit{Proof of PL inequality:} The \eqref{eqn: pl-inequality 1} is a consequence of strong convexity. Using strong convexity w.r.t. $\on\theta$ Lemma \ref{lem:lipschitz/strong convex},\eqref{eqn: hessian bound on-manifold} we have:
\begin{align*}
 \cl(\on\theta,\iter{\off\theta}{0})&\geq\cl(\iter{\on\theta}{t},\iter{\off\theta}{0})+\nabla_{\on\theta}\cl(\iter{\on\theta}{t},\iter{\off\theta}{0})^T(\on\theta-\iter{\on\theta}{t})+\frac{c_l\cdot\on{\sigma^2}}{2}\snorm{\on\theta-\iter{\on\theta}{t}}^2\\
\intertext{Minimizing both sides w.r.t $\on\theta$, happens for $\on\theta=\iter{\on\theta}{t}-\nabla_{\on\theta}\cl(\iter{\on\theta}{t},\iter{\off\theta}{0}){(c_l\on\sigma^2)}^{-1}$}
\min\limits_{\on\theta}\cl(\on\theta,\iter{\off\theta}{0})&\geq\cl(\iter{\on\theta}{t},\iter{\off\theta}{0})-\snorm{\nabla_{\on\theta}\cl(\iter{\on\theta}{t},\iter{\off\theta}{0})}^2{(2c_l\on{\sigma}^2)}^{-1} \\
\intertext{From arguments like $(i)$ in proof of lemma \ref{lem: coordinate wise optimality}, we know $c'_p\snorm{\nabla_{\on\theta}\cl(\iter{\on\theta}{t},\iter{\off\theta}{0})}^2=\snorm{\nabla_{\on\theta}\cl(\iter{\theta}{t})}^2$ for some proportionality constant $c'_p$}
\min\limits_{\on\theta}\cl(\on\theta,\iter{\off\theta}{0})&\geq\cl(\iter{\on\theta}{t},\iter{\off\theta}{0})-\snorm{\nabla_{\on\theta}\cl(\iter{\theta}{t})}^2{(2c'_pc_l\on{\sigma}^2)}^{-1}\qed
\end{align*}

Hence, from the progressive bound, we have \eqref{eqn: pb on component 2} :
\begin{align*}
    \cl(\iter{\on\theta}{t+1},\iter{\off\theta}{0})&\leq \cl(\iter{\on\theta}{t},\iter{\off\theta}{0}) - c_p(2c_u\on\sigma^2)^{-1}\snorm{\nabla_{\on\theta}\cl(\iter{\theta}{t})}^2\\
    \intertext{Subtracting $\min\limits_{\on\theta}\cl(\on\theta,\iter{\off\theta}{0})$ from both sides}
    \cl(\iter{\on\theta}{t+1},\iter{\off\theta}{0})-\min\limits_{\on\theta}\cl(\on\theta,\iter{\off\theta}{0})&\leq \cl(\iter{\on\theta}{t},\iter{\off\theta}{0})-\min\limits_{\on\theta}\cl(\on\theta,\iter{\off\theta}{0}) - c_p(2c_u\on\sigma^2)^{-1}\snorm{\nabla_{\on\theta}\cl(\iter{\theta}{t})}^2\\
    \intertext{Using \eqref{eqn: pl-inequality 2}}
    \cl(\iter{\on\theta}{t+1},\iter{\off\theta}{0})-\min\limits_{\on\theta}\cl(\on\theta,\iter{\off\theta}{0})&\leq \left(\cl(\iter{\on\theta}{t},\iter{\off\theta}{0})-\min\limits_{\on\theta}\cl(\on\theta,\iter{\off\theta}{0})\right)\cdot\left(1-\frac{c_pc_lc'_p}{c_u}\right)
\end{align*}
Suppose we have $T$ total iterations; then, inductively, the equations accumulate over as:
\begin{align*}
    \cl(\iter{\on\theta}{T},\iter{\off\theta}{0})-\min\limits_{\on\theta}\cl(\on\theta,\iter{\off\theta}{0})&\leq\left(\cl(\iter{\theta}{0})-\min\limits_{\on\theta}\cl(\on\theta,\iter{\off\theta}{0})\right)\left(1-\frac{c_pc_lc'_p}{c_u}\right)^T\\
    &\leq\left(\cl(\iter{\theta}{0})-\min\limits_{\on\theta}\cl(\on\theta,\iter{\off\theta}{0})\right)\exp\left(-T\frac{c_pc_lc'_p}{c_u}\right)\\
    &\leq \left(\cl(\iter{\theta}{0})-\cl(\theta^*)\right)\exp\left(-T\frac{c_pc_lc'_p}{c_u}\right)
\end{align*}
Hence, we require $T\geq \frac{c_u}{c_pc_lc'_p}\log\left(\left(\cl(\iter{\theta}{0})-\cl(\theta^*)\right)(\delta-C)^{-1}\right)$ for $\cl(\iter{\on\theta}{T},\iter{\off\theta}{0})-\min\limits_{\on\theta}\cl(\on\theta,\iter{\off\theta}{0})<\delta-C$. Hence, validating the series of equations \eqref{eqn: loss bound >C} and implying $\cl(\iter{\theta}{T})-\cl(\theta^*)\leq \delta$.
This gives us the rate $r_1$ in theorem \ref{thm: loss convergence}. Note that the rate $r_2$ holds regardless of $\delta$ value. The final rate is the minimum of the two rates.
\end{proof}
\section{Additional Experimental Details}
\label{sec: exp details}
A learning rate of $10^{-3}$ with default ADAM parameters were used for clean training. For generating PGD attacks a step size of $\frac{2}{255}$ with $\floor{\epsilon\cdot\frac{255}{2}}$ ($\epsilon$ is the attack strength) attack iterations and 1 restart was used. For KFAC preconditioner, the default hyper-parameters were used.\\
\textbf{Compute resource}: A single NVIDIA V100 gpu was used, requiring 2hrs, 3.5 hrs (MNIST, FMNIST 1000 epochs) per run for ADAM and ADAM+KFAC respectively. For CIFAR10 ADAM experiments it took ~6 hrs per run (4000 epochs). Attack time included.
\begin{table}[!ht]
    \centering
    \begin{tabular}{c}
         Layers\\\hline
       Conv2d(1, 16, 4, stride=2, padding=1), ReLU\\
      Conv2d(16, 32, 4, stride=2, padding=1), ReLU\\
       Linear(32*7*7,100), ReLU\\
        Linear(100, 10)\\
    \end{tabular}
    \caption{NN architecture FMNIST/MNIST.}
    \label{tab:standard architecturel}
\end{table}

\begin{table}[!ht]
    \centering
    \begin{tabular}{c}
         Layers\\\hline
       Conv2d(1, 16, 4, stride=2, padding=1), BatchNorm2d(16), ReLU\\
      Conv2d(16, 32, 4, stride=2, padding=1), BatchNorm2d(32), ReLU\\
       Linear(32*7*7,100), BatchNorm2d(100), ReLU\\
        Linear(100, 10)\\
    \end{tabular}
    \caption{NN architecture FMNIST/MNIST with BN.}
    \label{tab:bn architecturel}
\end{table}
\begin{table}[!ht]
    \centering
    \begin{tabular}{c}
         Layers\\\hline
       Conv2d(3, 128, 5, padding=2), ReLU\\
    Conv2d(128, 128, 5, padding=2),ReLU,\\
MaxPool2d(2,2), Conv2d(128, 256, 3, padding=1),ReLU,\\
Conv2d(256, 256, 3, padding=1),ReLU,\\
MaxPool2d(2,2), Flatten(),Linear(256*8*8,1024),ReLU,\\
Linear(1024, 512),ReLU\\
Linear(512, 10)
    \end{tabular}
    \caption{NN architecture CIFAR10.}
    \label{tab:cifar architecturel}
\end{table}

\begin{table}[!ht]
    \centering
    \begin{tabular}{c}
         Layers\\\hline
       Conv2d(3, 128, 5, padding=2), ReLU\\
       BatchNorm2d(128)\\
    Conv2d(128, 128, 5, padding=2),ReLU,\\
    BatchNorm2d(128),\\
MaxPool2d(2,2), Conv2d(128, 256, 3, padding=1),ReLU,\\
BatchNorm2d(256),\\
Conv2d(256, 256, 3, padding=1),ReLU,\\
BatchNorm2d(256),\\
MaxPool2d(2,2), Flatten(),Linear(256*8*8,1024),ReLU,\\
BatchNorm1d(1024)\\
Linear(1024, 512),ReLU\\
BatchNorm1d(512)\\
Linear(512, 10)
    \end{tabular}
    \caption{NN architecture CIFAR10 with BN.}
    \label{tab:cifar bn architecturel}
\end{table}

% \begin{figure}[ht]
%     \centering
%     \includegraphics[width=0.5\textwidth]{paper_figs/cifar10_bn.png}
%     \caption{PGD -$\ell_\infty$ robustness for clean CIFAR10 with batch-norm layers.}
%     \label{fig:batchnorm2}
% \end{figure}
\end{document}